\documentclass{article}

 \usepackage[preprint]{neurips_2026}

\usepackage[utf8]{inputenc} % allow utf-8 input
\usepackage[T1]{fontenc}    % use 8-bit T1 fonts
\usepackage{hyperref}       % hyperlinks
\usepackage{url}            % simple URL typesetting
\usepackage{booktabs}       % professional-quality tables
\usepackage{amsfonts}       % blackboard math symbols
\usepackage{nicefrac}       % compact symbols for 1/2, etc.
\usepackage{microtype}      % microtypography
\usepackage{xcolor}         % colors
\usepackage{amsmath}
\usepackage{amssymb}
\usepackage{amsthm}
\usepackage{longtable}
\usepackage{graphicx}
\usepackage{algorithm}
\usepackage{algpseudocode}
\usepackage{multirow}
\usepackage{colortbl}
\usepackage{subcaption}

\newtheorem{theorem}{Theorem}
\newtheorem{corollary}{Corollary}
\newtheorem{proposition}{Proposition}
\newtheorem{lemma}[theorem]{Lemma}

\title{Cluster Frequency Conformal Prediction for Local Coverage}

\author{%
  Tomer Lavi \\
  Institute of Applied AI Research (AAIR)\\
  Faculty of Computer and Information Science\\
  Ben-Gurion University of the Negev, Beer-Sheva 84105, ISRAEL\\
  \texttt{tomerlav@post.bgu.ac.il} \\
  \And
  Bracha Shapira \\
  Institute of Applied AI Research (AAIR)\\
  Faculty of Computer and Information Science\\
  Ben-Gurion University of the Negev, Beer-Sheva 84105, ISRAEL\\
  \texttt{bshapira@bgu.ac.il} \\
  \And
  Nadav Rappoport \\
  Institute for Interdisciplinary Computational Science (ICS)\\
  Faculty of Computer and Information Science\\
  Ben-Gurion University of the Negev, Beer-Sheva 84105, ISRAEL\\
  \texttt{nadavrap@bgu.ac.il} \\
}

\begin{document}

\maketitle

\begin{abstract}
Conformal prediction provides distribution-free coverage guarantees, but in many-class classification it may still under-cover specific classes or subpopulations, preventing safe deployment in high-stakes applications. We propose Cluster Frequency Conformal Prediction (CFCP), a plug-in framework that adapts conformal prediction to local structure in a learned representation space. CFCP clusters learned embeddings, estimates cluster-level label-frequency distributions from calibration data, and for each test point constructs a sample-specific probability vector by softly mixing nearby cluster distributions regularized with global-prior and reliability-aware shrinkage. This vector is then conformalized using standard set constructors. In the disjoint-split regime, CFCP inherits standard finite-sample marginal validity. Under additional assumptions, CFCP further admits a local-validity interpretation. Since representation clusters aggregate locally similar samples, their empirical class frequencies provide a stable estimate of local label ambiguity. Across image and text benchmarks, CFCP achieves the best class coverage in 15/16 dataset/score-family comparisons and a competitive prediction set size efficiency, with several settings substantially more efficient. Overall, our results show that cluster-frequency information provides an effective localized signal for improving classwise reliability in many-class conformal prediction.
\end{abstract}

\section{Introduction}

Conformal prediction (CP) provides finite-sample marginal coverage guarantees for set-valued predictions under exchangeability, a principled and widely adopted tool for uncertainty quantification in machine learning \citep{vovk2005algorithmic, angelopoulos2021gentle}. Marginal validity alone, however, can mask substantial disparities: a method may attain the target overall coverage while still under-covering specific classes or subpopulations. This issue is especially pronounced in many-class settings, where calibration data are sparse per class and the conditional label distribution may vary substantially across regions of the representation space.

In these settings, a single global prediction rule may be too coarse. Classwise or Mondrian conformal methods calibrate separately per label, but in many-class problems they can become unstable or overly conservative because each class receives only a limited number of calibration examples \citep{vovk2012conditional, ding2023manyclasses}. 
A complementary direction localizes conformal prediction in feature space: neighborhood-based methods \citep{ghosh2023ncp}, subgroup mixture models \citep{zhang2024pcp}, and covariate-aware reweighting schemes \citep{barber2023local, gil2024identifying}. All show that representation-aware adaptation can improve uncertainty quantification beyond globally calibrated procedures. Yet, to the best of our knowledge, no existing method directly uses local cluster-level label frequency distributions as the primary probabilistic signal for conformal set construction - the precise gap that this work addresses.

We build on the observation that learned representations often organize data into subpopulations with different label compositions (see Figure \ref{fig:Clusters.Horses}). Learned embeddings preserve the covariate geometry induced by the trained network and can capture within-class and cross-class subpopulations with distinct empirical label frequencies. A test point may lie in a region whose local ambiguity differs markedly from the global average, even when its base-model confidence appears similar to that of other points. This suggests that conformal prediction should adapt not only to a model's output scores, but also to the local label structure of the point's neighborhood in representation space. Such structure may reflect not only class-level ambiguity, but also finer within-class variation. This suggests a concrete design principle: rather than relying on a single global probability vector, a conformal method should estimate the local label distribution of a test point's neighborhood and use it directly as the input to set construction.

\begin{figure}[!ht]
    \centering
    \includegraphics[width=1\linewidth]{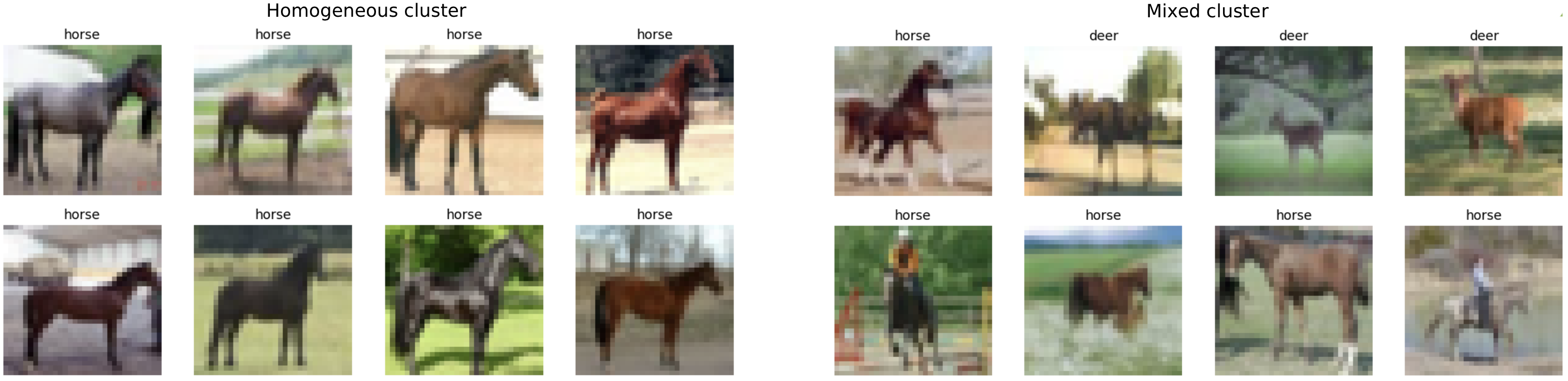}
    \caption{Visual intuition for our proposed Cluster Frequency Conformal Prediction (CFCP) framework. Example clusters induced by learned image representations. Left: eight samples from a highly homogeneous cluster whose sampled members are all labeled horse. Right: eight samples of a label-heterogeneous cluster. Such clusters illustrate the CFCP intuition that nearby samples in representation space may exhibit different local label distributions, which CFCP uses to form sample-specific probability vectors before standard conformal set construction.}
    \label{fig:Clusters.Horses}
\end{figure}

We therefore propose Cluster Frequency Conformal Prediction (CFCP), a plug-in conformal framework for classification. CFCP first clusters learned embeddings, then estimates cluster-level label frequency distributions from calibration data. For a new point, it retrieves nearby clusters, forms a soft mixture of their label distributions and regularizes that mixture with global priors and reliability-aware shrinkage. The resulting sample-specific probability vector is then passed to a standard conformal set constructor. Thus, CFCP preserves the standard conformal thresholding machinery while replacing the global probability vector with one that is locally adapted to representation-space structure. Interpretability is a natural byproduct: each prediction set is grounded in the label distribution of structurally similar training examples, making the uncertainty transparent and auditable.

This design also yields a clean validity interpretation. In the disjoint-split regime, where the local cluster-frequency model is estimated on one subset and the conformal threshold is calibrated on another, CFCP inherits the usual finite-sample marginal validity guarantee of split conformal prediction. Under additional assumptions, randomized Adaptive Prediction Sets (APS) \citep{romano2020classification} further admits a sharper local-validity interpretation. Our primary goal is practical: to improve classwise reliability in the many-class settings where purely global methods systematically hide coverage disparities, a regime where exact conditional validity guarantees are generally unattainable.

Empirically, CFCP is designed for the regime where classwise reliability is hardest: many classes, heterogeneous representation geometry and limited calibration data per class. Across image and text benchmarks, CFCP improves the fraction of classes that satisfy the target coverage level in 15/16 within score-family comparisons against strong baselines, while maintaining competitive prediction-set sizes and, in several settings, substantially improving efficiency. This suggests that local cluster-frequency information can improve the tradeoff between classwise coverage and practical prediction-set size.

Our contributions are:
\begin{itemize}
    \item We introduce CFCP, a plug-in conformal framework that constructs a sample-specific label distribution from cluster-level label frequencies in a learned representation space. CFCP provides both a computational and a statistical advantage over point-neighborhood methods in the many-class, high-dimensional regime.
    \item We show that CFCP inherits standard finite-sample marginal validity in the disjoint-split regime and we provide an oracle local-validity interpretation for randomized APS.
    \item We empirically show on challenging many-class image and text benchmarks that CFCP often improves classwise coverage while remaining competitive in set size.
\end{itemize}

\subsection{Background and Related Work}

Let \(\{(X_i,Y_i)\}_{i=1}^{m}\) be exchangeable, let a base predictor be trained on a proper training split and let a nonconformity score \(s:\mathcal{X}\times\mathcal{Y}\to\mathbb{R}\) quantify how atypical a candidate label is for a feature vector. In split conformal prediction (Split), calibration scores are computed on a held-out set \(s_i=s(X_i,Y_i)\) for \(i\in|m|\). Sort them as \(s_{(1)}\le\cdots\le s_{(m)}\) and set
\(
\hat q_{1-\alpha}\;=\;s_{(k)},\qquad k\;=\;\left\lceil (m+1)(1-\alpha)\right\rceil.
\)
For a new \(x\), the conformal prediction set is
\(
\hat{c}(x)\;=\;\{\,y\in\mathcal{Y}: s(x,y)\le \hat q_{1-\alpha}\,\}.
\)
Then, the marginal coverage is \(\mathbb{P}\{Y\in\hat{c}(X)\}\ge 1-\alpha\).

Least Ambiguous Classifier (LAC) applies Split to the softmax score $1-\hat p_x(y)$. It returns the set of labels whose predicted probability is at least the calibrated threshold $1-\hat q_{1-\alpha}$. Adaptive set constructions such as APS, regularized APS (RAPS) and Sorted Adaptive Prediction Sets (SAPS) improve efficiency by using cumulative probability scores and rank-dependent regularization (\cite{romano2020classification, huang2023conformal}). These methods are strong baselines, but they primarily target marginal validity and set size, not disparities in coverage across classes or subpopulations.

To address this, class-conditional and many-class conformal methods calibrate more finely across labels. Mondrian conformal prediction and Inductive Conformal Prediction (ICP) target label-conditional behavior \citep{vovk2012conditional}, while \cite{ding2023manyclasses} stabilize many-class classwise calibration by clustering labels with similar conformity-score distributions (CCP). \cite{ding2025conformal} further study classwise reliability in long-tailed classification by reweighting predicted probabilities of under-represented classes. \cite{shi2024rc3p} follows CCP while improving efficiency by classwise label ranking (RC3P).

A complementary direction localizes conformal prediction in feature space. Neighborhood CP (NCP) uses nearby examples in a learned embedding space \citep{ghosh2023ncp}. Posterior CP (PCP) models conformity scores through latent subgroup mixtures \citep{zhang2024pcp}, and recent work on local weighting and subgroup discovery further highlights the value of covariate-aware uncertainty quantification \citep{barber2023local, gil2024identifying}.

CFCP is most closely related to NCP and the local conformal literature, but differs fundamentally in both mechanism and robustness. NCP reweights the conformal quantile using per-point k-NN distances in the raw embedding space. This approach (i) scales as $O(n_{cal}\cdot d)$ per query, (ii) degrades in high dimensions as pairwise Euclidean distances concentrate and weights become nearly uniform, and (iii) provides no fallback when the neighborhood is uninformative. CFCP instead fits spherical k-means on $\ell2$-normalized embeddings, compresses the calibration set into $K$ centroids, and constructs a sample-specific label-frequency distribution, not a reweighted quantile, and reliability-aware shrinkage toward a global prior. This design reduces inference from $O(n_{cal}\cdot d)$ to $O(K\cdot d)$, exploits the angular geometry that supervised training imposes on the embedding space and gracefully degrades to standard conformal prediction when cluster assignments are diffuse.

\section{Method}
\label{sec:method}

\begin{figure}[!ht]
    \centering
    \includegraphics[width=1\linewidth]{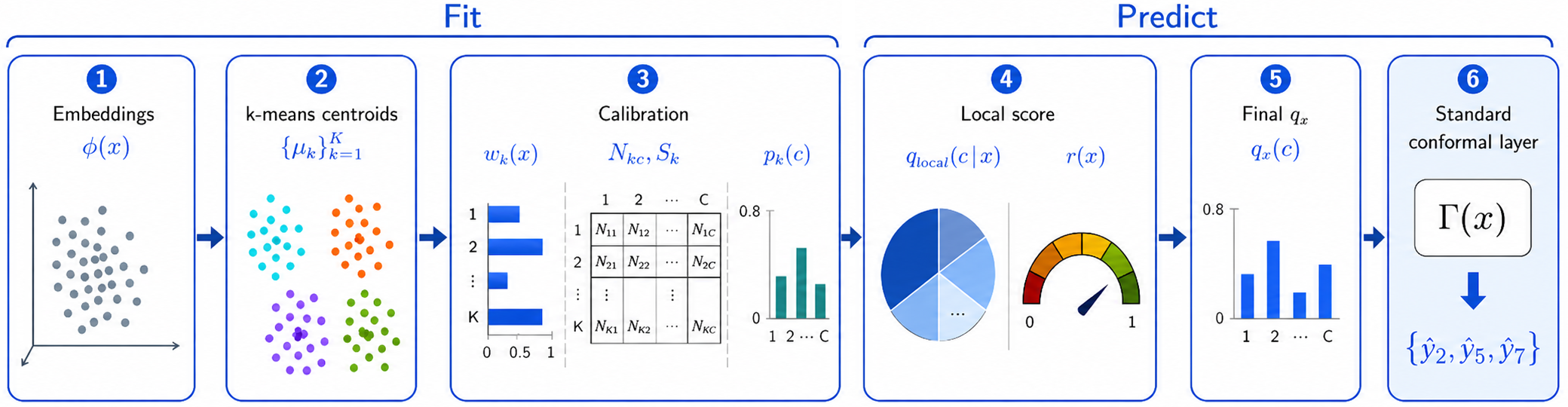}
\caption{
CFCP pipeline. Learned representations are clustered (1, 2) and calibration labels are accumulated as soft cluster counts  (3). Smoothed cluster distributions are mixed according to the test point's soft assignment weights and reliability-shrunk to form probabilities (4), which are passed to a standard conformal set constructor (5+6). See full notation list in Table \ref{tab:cfcp-symbols} in the appendix.
}
\label{fig:cfcp_pipeline}
\end{figure}

CFCP has one core design choice: replace the global class-probability vector used by standard conformal set constructors with a locally adapted vector $q_x$, estimated from cluster-level label frequencies in representation space. This section formalizes that construction.

The method first clusters normalized learned representations and uses soft nearest-centroid assignments to estimate cluster-level label-frequency statistics from the calibration subset (Figure~\ref{fig:cfcp_pipeline}). These statistics are smoothed toward a global prior to obtain cluster distributions, which are then mixed according to the test point's soft assignment weights. To avoid diffuse assignments or weakly supported clusters, CFCP applies reliability-aware shrinkage toward a fallback prior, producing the final sample-specific vector $q_x$. This vector is then passed to a standard conformal set constructor such as LAC, APS, RAPS, or SAPS.

\subsection{Local cluster-frequency model}

Let $f$ be a pre-trained classifier over $C$ classes, and let $\phi(x)\in\mathbb{R}^d$ be a fixed learned representation of input $x$. We $\ell_2$-normalize all embeddings and fit spherical $k$-means on the training representations, obtaining centroids $\mu_1,\dots,\mu_K\in\mathbb{R}^d$. For a point $x$ with normalized embedding $z$, we identify its top-$m$ nearest centroids by cosine similarity and convert their similarities $\delta_1(x),\dots,\delta_m(x)$ into soft assignment weights where $\tau>0$ controls assignment sharpness. Lower values yield more local assignments.
\begin{equation}
w_j(x)=\frac{\exp(\delta_j(x)/\tau)}{\sum_{\ell=1}^m \exp(\delta_\ell(x)/\tau)}, \qquad j=1,\dots,m,
\end{equation}
Let the calibration set be split into $D_{\mathrm{stat}}$ and $D_q$, where $D_{\mathrm{stat}}$ is used to estimate local label distributions and $D_q$ is reserved for conformal quantile calibration. Rather than assigning each calibration point to a single cluster, CFCP distributes it fractionally across its $m$ nearest centroids. Formally, the soft cluster-label counts are:
\begin{equation}
N_{kc}=\sum_{(x_i,y_i)\in D_{\mathrm{stat}}}\sum_{j=1}^m w_{ij}\,\mathbf{1}\{k_j(x_i)=k,\; y_i=c\},
\end{equation}
together with effective cluster supports:
\begin{equation}
S_k=\sum_{(x_i,y_i)\in D_{\mathrm{stat}}}\sum_{j=1}^m w_{ij}\,\mathbf{1}\{k_j(x_i)=k\}.
\end{equation}
Thus, each calibration point contributes fractionally to several nearby clusters rather than to a single hard partition.

For each cluster $k$, we estimate a label distribution by shrinking the soft counts toward a global prior $\pi\in\Delta^{C-1}$:
\begin{equation}
\hat p_k(c)=\frac{N_{kc}+\beta \pi_c}{S_k+\beta}, \qquad c=1,\dots,C,
\end{equation}
where $\beta>0$ is a smoothing parameter. For a test point $x$, this yields the local cluster-frequency mixture
\begin{equation}
q_x^{\mathrm{local}}(c)=\sum_{j=1}^m w_j(x)\hat p_{k_j(x)}(c), \qquad c=1,\dots,C.
\end{equation}
Because this local mixture may be unreliable when assignments are diffuse or nearby clusters have low support, we further shrink it toward a fallback prior $\pi_x$, taken as the base model's per-sample class-probability vector. We define a reliability weight:
$\mathrm{conf}(x)=\Big(\max_{1\le j\le m} w_j(x)\Big)^\gamma, \mathrm{sup}(x)=\sum_{j=1}^m w_j(x) S_{k_j(x)},\, \mathrm{supw}(x)=\frac{\mathrm{sup}(x)}{\mathrm{sup}(x)+\beta_{\mathrm{sup}}},\, r(x)=\mathrm{conf}(x)\cdot\mathrm{supw}(x)\in[0,1],$ 
where $\gamma>0$ controls the effect of assignment concentration and $\beta_{\mathrm{sup}}>0$ regularizes the neighborhood support term. $r(x)$ increases when the nearest centroid is dominant and those centroids are well‑supported, otherwise we back off to the base model. The final CFCP probability vector is
\begin{equation}
q_x(c)=r(x)\,q_x^{\mathrm{local}}(c) + (1-r(x))\,\pi_x(c), \qquad c=1,\dots,C.
\end{equation}

\subsection{Conformal prediction}

CFCP separates local probability estimation from set construction. After computing $q_x$, we apply a standard split-conformal set constructor such as LAC, APS, RAPS or SAPS. Let $D_q=\{(x_i,y_i)\}_{i=1}^{n_q}$ denote the calibration subset reserved for quantile estimation, and let $s_i=s(x_i,y_i)$ be the induced nonconformity scores under the chosen constructor. The split-conformal threshold is
\(
\hat q_{1-\alpha}=S_{(k)}, k=\left\lceil (n_q+1)(1-\alpha)\right\rceil,
\)
where $S_{(k)}$ is the $k$-th order statistic.

For the Split score, the resulting prediction set is
\(
\hat\Gamma(x)=\{c\in\{1,\dots,C\}: q_x(c)\ge 1-\hat q_{1-\alpha}\}.
\)\\
For APS, RAPS and SAPS, labels are sorted by $q_x(c)$ and included according to the corresponding cumulative or rank score rule. Hence, CFCP is a plug-in conformal framework: the only modification relative to standard conformal prediction is the construction of the probability vector $q_x$ from local cluster-frequency information. Full score definitions for each constructor are given in Appendix \ref{app:method}.

In the disjoint-split regime, where $D_{\mathrm{stat}}$ is used only to estimate the local model and $D_q$ only for threshold calibration, CFCP inherits the standard finite-sample marginal validity guarantee. The formal theorem and proof are deferred to Appendix \ref{app:validity}.

\section{Theoretical Analysis}
\label{sec:theoretical-analysis}
CFCP is a plug-in conformal framework: it first estimates a local probability vector $q_x$ from calibration-derived cluster statistics, and then applies a standard split-conformal threshold on a held-out calibration subset. We therefore distinguish between two levels of validity. First, in the disjoint-split regime $D_{\mathrm{stat}} \cap D_q = \emptyset,$ CFCP inherits the standard finite-sample marginal validity guarantee of split conformal prediction. Second, under additional assumptions, randomized APS admits an exact local validity statement. The oracle intuition follows the randomized oracle construction of \cite{romano2020classification}, while the set-valued oracle perspective is also consistent with \citet{sadinle2019least}. More broadly, this discussion fits the literature on conditional and local conformal validity, where exact local guarantees are known to require stronger structure or randomization
\citep{vovk2012conditional,barber2023local}.

Appendix \ref{app:validity} provides the formal proofs for the following theorems and corollaries, the conditional-uniformity lemma underlying the oracle APS result and an approximate local-validity proposition for the general CFCP framework. 

\begin{theorem}[Finite-sample marginal validity of disjoint-split CFCP]
\label{thm:cfcp-marginal}
Assume the disjoint-split CFCP construction above: the local CFCP model is fitted without using labels from $D_q$, the calibration examples in $D_q$ and the test point are exchangeable conditional on the fitted local model $\widehat M$, and any auxiliary randomization used by APS/RAPS/SAPS is independent. Let $\widehat{\Gamma}$ be the resulting CFCP prediction set obtained with the split-conformal threshold calibrated on $D_q$. Then
\[
\Pr\{Y_{n+1}\in \widehat{\Gamma}(X_{n+1})\mid \widehat M\}\ge 1-\alpha .
\]
\end{theorem}

\begin{corollary}[Marginal validity]
\label{cor:cfcp-marginal}
Under the assumptions of Theorem~\ref{thm:cfcp-marginal},
\[
\Pr\!\left\{Y_{n+1}\in \widehat\Gamma(X_{n+1})\right\}\ge 1-\alpha .
\]
\end{corollary}

Theorem \ref{thm:cfcp-marginal} is the exact finite-sample guarantee for CFCP in the clean disjoint-split regime. Since CFCP uses sample-dependent soft mixtures over neighboring clusters rather than hard cluster-specific quantiles, we do not claim an exact clusterwise coverage theorem in general. Nevertheless, randomized APS admits a sharper oracle statement. If the final CFCP vector recovers the true conditional law, $q_x(c)=\Pr(Y=c\mid X=x)$, then the randomized APS true-label score is conditionally uniform given $X=x$, which yields exact local validity for any measurable region defined independently of the calibration labels.

\begin{corollary}[Oracle local validity for randomized APS]
\label{cor:cfcp-oracle-local-main}
Assume the setting of Theorem~\ref{thm:cfcp-marginal}, and suppose additionally that CFCP recovers the true conditional label law, i.e. $q_x(c)=\Pr(Y=c\mid X=x)$ for all $x,c$, and that the set constructor is randomized APS. Then for every measurable region $A\subseteq\mathcal X$ defined without using labels from $D_q$,
\[
\Pr\{Y_{n+1}\in \hat\Gamma(X_{n+1}) \mid X_{n+1}\in A\}\ge 1-\alpha.
\]
\end{corollary}
A particularly relevant special case is the piecewise-constant cluster oracle: if clusters recover regions on which the conditional label distribution is constant and CFCP recovers this distribution, the same local-validity guarantee holds on each oracle cluster. See Corollary~\ref{cor:cluster-oracle} in the appendix.

\paragraph{Oracle efficiency intuition.}
Beyond validity, CFCP also has a natural oracle efficiency intuition. Let $\mathcal G_{\mathrm{glob}}$ denote the information available to a global predictor and let $\mathcal G_{\mathrm{loc}}$ denote the refined information obtained by augmenting it with the oracle neighborhood / cluster information used to construct $q_x$, with $\mathcal G_{\mathrm{glob}} \subseteq \mathcal G_{\mathrm{loc}}$, every
$\mathcal G_{\mathrm{glob}}$-measurable set predictor is also $\mathcal G_{\mathrm{loc}}$-measurable. Hence
\(
\mathcal A(\mathcal G_{\mathrm{glob}},\alpha)
\subseteq
\mathcal A(\mathcal G_{\mathrm{loc}},\alpha).
\) 
Define
\(
V(\mathcal G,\alpha)
=
\inf_{\Gamma:\,\mathbb P(Y \in \Gamma(X)\mid \mathcal G)\ge 1-\alpha}
\mathbb E[|\Gamma(X)|].
\)
Since every global predictor is also admissible under the richer information set, it follows that
\(
V(\mathcal G_{\mathrm{loc}},\alpha)\le V(\mathcal G_{\mathrm{glob}},\alpha).
\)\\
Consequently, among predictors satisfying the same target coverage constraint, the expected set size of the population-optimal neighborhood-aware predictor is less than or equal to the population-optimal global predictor. 
Equivalently, at a fixed expected set-size budget, the best neighborhood-aware predictor attainable coverage is higher or equal. Define
\(
U(\mathcal G,m)
:=
\sup_{\Gamma:\ \Gamma(X)\ \mathcal G\text{-measurable},\ \mathbb E[|\Gamma(X)|]\le m}
\mathbb P\!\bigl(Y \in \Gamma(X)\bigr).
\)
Then
\(
U(\mathcal G_{\mathrm{loc}},m)
\;\ge\;
U(\mathcal G_{\mathrm{glob}},m).
\)

This is a population-level motivation rather than a finite-sample dominance claim for CFCP because the practical method uses estimated cluster frequencies, reliability shrinkage and a single split-conformal threshold. Still, it explains why local cluster-frequency information improves efficiency in practice.

\subsection{Computational complexity}
\label{subsec:complexity}

CFCP is designed to scale to large embedding sets. Since the thresholding stage uses a standard split-conformal procedure such as LAC, APS, RAPS or SAPS, its computational overhead beyond ordinary conformal prediction is concentrated in the representation-space localization step, namely clustering and nearest-centroid retrieval. Complexity details are given in Appendix \ref{app:complexity}.

\section{Experiments}
\label{sec:experiments}

We evaluate CFCP through extensive experiments, comparing it with strong baselines across a diverse set of classification tasks and model architectures. Code for reproducing the experiments and is available at \url{https://anonymous.4open.science/r/CFCP-C039/}. 

\subsection{Experimental setup}
This section describes the main components of the experimental setup. Full details on model training and experimental setup are available in Appendix \ref{app:experimental_setup}.

\textbf{Datasets}: We selected datasets in the many-class regime, where each class is supported by relatively few examples. CIFAR-100 (60,000 images, 100 classes) (\cite{Krizhevsky09learningmultiple}) and ImageNet ($\approx$1.33M images, 1000 classes, imbalanced) (\cite{deng2009imagenet}) for image classification. ImageNet-V2 out-of-sample test set for ImageNet (10,000 images equally partitioned on the ImageNet classes) (\cite{recht2019imagenet}). Web of Science (WOS-46985) ($\approx47k$ records, 134 classes, imbalanced) (\cite{kowsari2017hdltex}) for text classification. 

\textbf{Architectures}: Fine-tuned ResNet-50 (\cite{he2016deep}) for CIFAR. Fine-tuned DistilBERT (\cite{sanh2019distilbert}) on text data. For ImageNet we followed the training procedure of \cite{ding2023manyclasses}, using SimCLR representations (\cite{chen2020big}) and a single layer prediction head. 

\textbf{Learned Representations}: We follow \cite{yosinski2014transferable, chen2020big} to extract CNN activations and Transformer hidden states. 

\textbf{Baseline methods}: We compared our CFCP method to Split, ICP, CCP, NCP and RC3P by evaluating each on these score methods: LAC, APS, RAPS and SAPS. 

\textbf{Coverage}: $\alpha = 0.1$. 

\textbf{Uncertainty reporting}: All reported result values are means with 95\% CI over five repeated dataset splits with fixed seeds.

\subsection{Evaluation metrics}

\textbf{Class Coverage}: We report class coverage, defined as the number of classes whose empirical class-conditional coverage meets the target level $1 - \alpha$. Let $\widehat{\mathrm{Cov}}_c$ denote the empirical class-conditional coverage for class $c$. We define Class Coverage at level $\alpha$ as \(
    \widehat{\mathrm{COV}}
    =
    \frac{1}{C}
    \sum_{c=1}^{C}
    \mathbf{1}\!\left\{\widehat{\mathrm{Cov}}_c \ge 1-\alpha\right\}.
    \)
This metric counts the number of classes whose empirical coverage meets the nominal target.
Higher values are better, and $\widehat{\mathrm{COV}}=1$ iff all classes satisfy the target coverage level.

\textbf{Efficiency (Set Size)}: We measure efficiency by the average set size, namely the mean number of labels contained in the prediction set across test examples. 
    Let $\widehat{\Gamma}(x) \subseteq \{1,\dots,C\}$ denote the prediction set for input $x$. Its size is $|\widehat{\Gamma}(x)|$, the number of labels included in the set.
Over a test set $\{(x_i,y_i)\}_{i=1}^n$, we define the empirical average set size as
$
\widehat{\mathrm{Size}}=\frac{1}{n}\sum_{i=1}^n |\hat{\Gamma}(x_i)|
=\frac{1}{n}\sum_{i=1}^n\sum_{c=1}^C \mathbf{1}\{c\in\hat{\Gamma}(x_i)\}.
$
Lower values for identical coverage level indicate greater efficiency.

\textbf{Weighted under-coverage (WUC)}: Following the average class coverage gap (CovGap) definition (\cite{ding2023manyclasses}), we used its under-coverage variant (UnderCovGap) to evaluate the class conditional coverage. We preferred it over CovGap which penalizes methods which exhibit higher marginal coverage by considering over-coverage as an error.
Let $\widehat{\mathrm{Cov}}_c$ denote the empirical class-conditional coverage for class $c$, and let $w_c$ be the empirical prevalence of class $c$ in the test set, for $c=1,\dots,C$ with $\sum_{c=1}^C w_c = 1$. We define the weighted under-coverage at level $\alpha$ as
\(
\widehat{\mathrm{WUC}}_p
=
\sum_{c=1}^{C} w_c \Bigl[(1-\alpha)-\widehat{\mathrm{Cov}}_c\Bigr]_+^{\,p},
\)
where $[u]_+ = \max(u,0)$ and typically $p\in\{1,2\}$, where $p=1$ (used in our reports) gives a linear penalty and $p=2$ penalizes severe under-coverage more strongly.
This metric measures only coverage deficits relative to the target $1-\alpha$, and equals zero iff all classes satisfy the nominal coverage target.

\textbf{Max. coverage error (MaxCE)}: We report the maximum coverage error, which quantifies the worst under-covered class and therefore highlights whether gains in average classwise coverage mask severe failures on individual classes. Let $\widehat{\mathrm{Cov}}_c$ denote the empirical class-conditional coverage for class $c$. We define the maximum coverage error at level $\alpha$ as
\(
\widehat{\mathrm{MaxCE}}
=
\max_{c\in\{1,\dots,C\}}
\left[(1-\alpha)-\widehat{\mathrm{Cov}}_c\right]_+,
\)
where $[u]_+ = \max(u,0)$. This metric measures the worst classwise coverage shortfall relative to the nominal target. It is equal to zero iff all classes satisfy the target coverage level.

On our experiments we use $\widehat{\mathrm{COV}}$ as primary metric and $\widehat{\mathrm{Size}}$ as primary companion metric. Secondary diagnostics: $\widehat{\mathrm{WUC}}$ and $\widehat{\mathrm{MaxCE}}$. $\widehat{\mathrm{COV}}$ is a natural binary version of the class-conditional coverage criterion. A closely related metric,  Under Coverage Ratio (UCR), is used as the primary evaluation criterion in RC3P, and is equivalent to $1 - \widehat{\mathrm{COV}}$. \cite{ding2025conformal} similarly report the fraction of classes with coverage below a threshold (FracBelow), showing that threshold-crossing counts and average-gap metrics are treated as complementary rather than competing.

\subsection{Results}

\begin{table}[t]
\centering
\caption{Within-family comparison of CFCP against baselines at $\alpha=0.1$. Values are means over five repeated dataset splits. \textbf{Bold} marks the best class coverage within each family and, for set size, the lower value between CFCP and the best-coverage baseline. Full results including 95\% CIs are in Appendix~\ref{app:results}. }
\label{tab:main_results_cov}
\small
\setlength{\tabcolsep}{4pt}
\resizebox{\linewidth}{!}{%
\begin{tabular}{l l cc cc cc cc cc cc}
\toprule
 &  & \multicolumn{2}{c}{\textbf{CFCP (ours)}} & \multicolumn{2}{c}{Split} & \multicolumn{2}{c}{ICP} & \multicolumn{2}{c}{CCP} & \multicolumn{2}{c}{NCP} & \multicolumn{2}{c}{RC3P} \\
\cmidrule(lr){3-4}\cmidrule(lr){5-6}\cmidrule(lr){7-8}\cmidrule(lr){9-10}\cmidrule(lr){11-12}\cmidrule(lr){13-14}
\textbf{Dataset} & \textbf{Score} & Cov$\uparrow$ & Size$\downarrow$ & Cov$\uparrow$ & Size$\downarrow$ & Cov$\uparrow$ & Size$\downarrow$ & Cov$\uparrow$ & Size$\downarrow$ & Cov$\uparrow$ & Size$\downarrow$ & Cov$\uparrow$ & Size$\downarrow$ \\
\midrule
\multirow{4}{*}{CIFAR-100} & LAC & \textbf{0.592} & \textbf{7.9} & 0.516 & 8.0 & 0.568 & 11.2 & 0.534 & 9.7 & 0.566 & 8.9 & 0.568 & 11.2 \\
 & APS & \textbf{0.590} & \textbf{8.3} & 0.516 & 11.1 & 0.560 & 14.6 & 0.530 & 12.2 & 0.572 & 12.0 & 0.582 & 11.3 \\
 & RAPS & \textbf{0.598} & \textbf{9.1} & 0.556 & 8.3 & 0.570 & 11.6 & 0.580 & 10.6 & 0.562 & 7.9 & 0.564 & 11.6 \\
 & SAPS & \textbf{0.586} & \textbf{8.8} & 0.558 & 8.2 & 0.582 & 11.5 & 0.568 & 10.3 & 0.566 & 7.7 & 0.548 & 10.9 \\
\midrule
\multirow{4}{*}{ImageNet} & LAC & \textbf{0.508} & 1.9 & 0.499 & \textbf{1.8} & 0.493 & 3.9 & 0.494 & 3.3 & 0.494 & 1.9 & 0.425 & 3.5 \\
 & APS & \textbf{0.672} & \textbf{15.1} & 0.655 & 35.2 & 0.658 & 42.8 & 0.646 & 38.0 & 0.667 & 15.2 & 0.654 & 8.4 \\
 & RAPS & \textbf{0.589} & 3.6 & 0.553 & 3.3 & 0.546 & 5.5 & 0.536 & 8.5 & 0.565 & \textbf{2.8} & 0.546 & 5.9 \\
 & SAPS & \textbf{0.571} & 2.6 & 0.529 & 2.2 & 0.533 & 5.1 & 0.549 & 6.8 & 0.559 & \textbf{2.4} & 0.556 & 6.0 \\
\midrule
\multirow{4}{*}{ImageNet-V2} & LAC & 0.446 & 2.5 & \textbf{0.452} & \textbf{2.4} & 0.388 & 5.0 & 0.386 & 4.3 & 0.445 & 2.8 & 0.424 & 6.9 \\
 & APS & \textbf{0.714} & 31.2 & 0.710 & 64.7 & 0.708 & 72.0 & 0.707 & 69.2 & 0.704 & 39.8 & 0.712 & \textbf{28.1} \\
 & RAPS & \textbf{0.552} & \textbf{4.8} & 0.525 & 4.3 & 0.523 & 6.4 & 0.529 & 8.8 & 0.541 & 4.2 & 0.544 & 10.2 \\
 & SAPS & \textbf{0.465} & \textbf{3.0} & 0.458 & 2.8 & 0.463 & 5.5 & 0.455 & 7.4 & 0.446 & 2.8 & 0.440 & 6.0 \\
\midrule
\multirow{4}{*}{WOS-46985} & LAC & \textbf{0.645} & \textbf{2.5} & 0.622 & 1.2 & 0.621 & 4.1 & 0.601 & 5.4 & 0.619 & 1.3 & 0.634 & 3.8 \\
 & APS & \textbf{0.633} & \textbf{1.5} & 0.546 & 3.4 & 0.606 & 6.2 & 0.570 & 6.6 & 0.609 & 1.5 & 0.587 & 3.8 \\
 & RAPS & \textbf{0.633} & \textbf{1.5} & 0.560 & 2.8 & 0.609 & 5.9 & 0.587 & 6.0 & 0.600 & 1.5 & 0.579 & 4.0 \\
 & SAPS & \textbf{0.643} & 2.3 & 0.590 & 1.3 & 0.610 & 5.8 & 0.604 & 5.1 & 0.630 & \textbf{1.3} & 0.604 & 3.7 \\
\bottomrule
\end{tabular}%
}
\end{table}

Table~\ref{tab:main_results_cov} reports within score-family comparisons at $\alpha=0.1$. CFCP achieves the best class coverage in 15 of 16 dataset/score comparisons. The only exception is ImageNet-V2 with LAC, where Split is slightly better in class coverage. Overall, these results support the main goal of CFCP: improving the fraction of classes that satisfy the target coverage level while preserving the standard plug-in conformal set-construction framework.

The gains are consistent across modalities and score families. On CIFAR-100, ImageNet, and WOS-46985, CFCP obtains the best class coverage for all four score families. On ImageNet-V2, which is the hardest shifted test setting, CFCP obtains the best class coverage for APS, RAPS, and SAPS, and remains essentially tied with the best LAC baseline. These results suggest that local cluster-frequency information is useful not only for image classification, but also for text classification, and that the benefit is not tied to a particular conformal score.

The comparison with NCP is particularly informative because NCP is the closest local-neighborhood baseline and partly motivates CFCP. CFCP improves over NCP in class coverage in all 16 comparisons, supporting the methodological gap addressed here: locality alone, when implemented through point-wise neighborhood weighting of conformal scores, may be insufficient for classwise reliability. The advantage over NCP suggests that soft assignment to learned clusters, combined with global shrinkage when the local signal is unreliable, provides a more stable local signal than direct point-neighborhood weighting in high-dimensional many-class settings.

A second important finding is that CFCP's class-coverage gains often come with strong efficiency.  When comparing CFCP to the leading baseline in each dataset/score family, CFCP has the smaller average set size in 10 of 16 comparisons. Moreover, even when CFCP is not the smallest method, its set size is usually close to the most efficient competing baselines and often much smaller than the strongest class-coverage baseline. For example, on ImageNet with RAPS, CFCP improves class coverage from 0.565 for NCP and 0.546 for RC3P to 0.589, while keeping a small average set size of 3.6. Although NCP is smaller, CFCP substantially improves class coverage and remains more efficient than ICP, CCP and RC3P. On ImageNet-V2 with RAPS, CFCP achieves the best class coverage, 0.552, with a set size of 4.8, compared with 10.2 for RC3P, 6.4 for ICP, and 8.8 for CCP. It is only slightly larger than Split and NCP, which obtain lower class coverage. Thus, CFCP does not merely trade larger sets for higher class coverage. In many settings it improves class coverage while 
also producing competitive or more efficient prediction sets. 

This observation is also consistent with recent findings that post-hoc temperature scaling can improve calibration while increasing the size of adaptive conformal prediction sets, since smoother probability vectors require APS/RAPS-style methods to accumulate more labels \citep{xi2024does}. CFCP offers a different route: when the local cluster signal is reliable, its probability vector is driven primarily by empirical cluster frequencies.

The secondary diagnostics in Appendix Table~\ref{tab:main_results_err} reveal a more nuanced picture. CFCP is not uniformly best in weighted under-coverage or maximum coverage error. This is expected from the design of the method: CFCP uses soft mixtures of nearby cluster distributions together with a single global conformal threshold, so it improves broad classwise coverage but may still dilute rare or highly atypical classes. Consequently, CFCP can cover more classes overall while leaving a small number of hard classes with larger residual under-coverage. We therefore use WUC and MaxCE as diagnostics rather than primary selection criteria: they expose remaining tail failures, while class coverage and set size capture the main coverage-efficiency objective of the paper.

At the same time, the efficiency of CFCP gives the user room to operate at slightly higher marginal coverage when lower under-coverage is preferred. This is illustrated on ImageNet-V2, where increasing CFCP's marginal coverage by 1\% achieves the best WUC for both APS and RAPS while remaining markedly more efficient than the runner-up baselines. For APS, increasing marginal coverage by 1\% yields WUC 0.041 with 0.737 COV and 35.8 Size, compared with 0.043, 0.71 and 64.7 for the strongest baseline. For RAPS, the corresponding values are 0.086, 0.574 and 5.2, versus 0.087, 0.529 and 8.8. This shows that CFCP’s efficiency advantage is not only a convenience, but also a practical degree of freedom for selecting a better coverage-efficiency operating point.

Finally, all methods require substantially larger prediction sets on ImageNet-V2 than on ImageNet, indicating that the distribution shift in ImageNet-V2 makes confident prediction more difficult. However, even in this harder regime, CFCP preserves its main qualitative advantage: it improves or matches the best class coverage while often maintaining a clear efficiency benefit.

Overall, Table~\ref{tab:main_results_cov} supports the main empirical claim. CFCP improves class coverage in nearly all within score-family comparisons and often achieves this improvement with equal or smaller prediction sets than the leading baseline. These results indicate that cluster-frequency information provides a useful localized signal for conformal prediction in many-class classification.

\paragraph{Sensitivity and ablation}
Our sensitivity analysis indicates that CFCP is task-adaptive rather than brittle. While the best hyperparameters depend on the dataset and model, reflecting differences in representation geometry and local label structure, our hyperparameter searches show that performance remains stable across a compact region of the hyperparameter space. In particular, near-best configurations typically differ by only about 1-2 percentage points in class coverage while maintaining closely matched set sizes, suggesting that CFCP’s improvements do not hinge on a single finely tuned setting but arise from the underlying local cluster-frequency mechanism itself. This is consistent with the role of CFCP’s locality-specific components: clustering granularity, neighborhood size, assignment temperature, smoothing and reliability-based shrinkage, which are designed to adapt the method to the structure of each task. Table \ref{tab:hpo_grid} in the appendix shows a hyperparameter search example.

Ablations on ImageNet and ImageNet-V2 confirm that CFCP's gains come from the combination of learned clusters and soft cluster assignment. Replacing the soft neighborhood mixture with hard nearest-cluster assignment generally reduces class coverage and substantially increases set size, especially for RAPS and SAPS. Conversely, removing clusters and relying only on a global label-frequency model collapses the locality mechanism: although it can recover moderate class coverage, it does so by producing very large prediction sets. These results support the role of soft cluster-frequency mixtures as the key component behind CFCP's coverage-efficiency tradeoff. Full ablation results are reported in Table \ref{tab:ablation} in Appendix \ref{app:results}.

In practice, the locality-specific hyperparameters can be selected on an internal validation split, with the final conformal threshold calibrated on a separate held-out calibration subset.

\section{Discussion, limitations and future work}
\label{sec:dicussion}
Our results indicate that local cluster-frequency information can improve conformal prediction by making prediction sets adapt to representation-space subpopulations rather than to a single global label distribution. CFCP improves class coverage while remaining competitive in set size, and in several settings substantially more efficient. More broadly, CFCP may help at the level of within-class subpopulations when representation-space clusters capture meaningful local structure, though this remains a hypothesis beyond the experiments in this paper.

The main limitation is that smoothing across nearby clusters, together with a single global conformal threshold, can still dilute rare or atypical classes. As a result, CFCP may cover more classes overall while leaving a small number of hard classes with larger residual under-coverage or worse maximum coverage error. Appendix \ref{app:failure} provides a CIFAR-100 failure-mode case study illustrating this behavior.

Future work should therefore focus on tail-sensitive extensions, such as class or cluster-aware threshold adjustments, adapting the framework to tabular problems using tree-ensemble representations, and generalizing the idea to conformal prediction for continuous labels via localized interval construction.

\begin{ack}
Use unnumbered first level headings for the acknowledgments. All acknowledgments
go at the end of the paper before the list of references. Moreover, you are required to declare
funding (financial activities supporting the submitted work) and competing interests (related financial activities outside the submitted work).
More information about this disclosure can be found at: \url{https://neurips.cc/Conferences/2026/PaperInformation/FundingDisclosure}.

Do {\bf not} include this section in the anonymized submission, only in the final paper. You can use the \texttt{ack} environment provided in the style file to automatically hide this section in the anonymized submission.
\end{ack}

\newpage
% \section*{References}
\bibliographystyle{plainnat}
\bibliography{bibliography}

\newpage
%%%%%%%%%%%%%%%%%%%%%%%%%%%%%%%%%%%%%%%%%%%%%%%%%%%%%%%%%%%%

\appendix
\if 0
\section{Technical appendices and supplementary material}
Technical appendices with additional results, figures, graphs, and proofs may be submitted with the paper submission before the full submission deadline (see above). You can upload a ZIP file for videos or code, but do not upload a separate PDF file for the appendix. There is no page limit for the technical appendices. 

Note: Think of the appendix as ``optional reading'' for reviewers. The paper must be able to stand alone without the appendix; for example, adding critical experiments that support the main claims to an appendix is inappropriate. 
\fi
%%%%%%%%%%%%%%%%%%%%%%%%%%%%%%%%%%%%%%%%%%%%%%%%%%%%%%%%%%%%

\section{Additional Method Details}
\label{app:method}

\begin{table}
\centering
\caption{Notation used in the CFCP method.}
\label{tab:cfcp-symbols}
\begin{tabular}{@{}llp{0.58\linewidth}@{}}
\toprule
Symbol & Type & Definition \\
\midrule
\(x \in \mathcal{X}\) & input & Input example. \\
\(y,c \in \mathcal{Y}\) & label & True label \(y\) or candidate class \(c\). \\
\(C\) & scalar & Number of classes, \(|\mathcal{Y}|=C\). \\
\(f\) & model & Pre-trained classifier. \\
\(\phi(x) \in \mathbb{R}^d\) & vector & Learned representation of \(x\). \\
\(z\) & vector & \(\ell_2\)-normalized embedding of \(x\). \\
\(d\) & scalar & Embedding dimension. \\
\(K\) & scalar & Number of representation-space clusters. \\
\(\mu_k \in \mathbb{R}^d\) & vector & Centroid of cluster \(k\), for \(k=1,\ldots,K\). \\
\(m\) & scalar & Number of nearest centroids used for soft assignment. \\
\(\kappa_j(x)\) & index & Index of the \(j\)-th nearest centroid to \(x\). \\
\(\delta_j(x)\) & scalar & Cosine similarity between \(z\) and centroid \(\mu_{\kappa_j(x)}\). \\
\(w_j(x)\) & scalar & Soft assignment weight of \(x\) to its \(j\)-th nearest centroid. \\
\(\tau\) & scalar & Softmax temperature controlling assignment sharpness. Smaller values yield more local assignments. \\
\(\mathcal{D}_{\mathrm{stat}}\) & set & Calibration subset used to estimate cluster-level label statistics. \\
\(\mathcal{D}_{q}\) & set & Held-out calibration subset used to estimate the conformal threshold. \\
\(N_{kc}\) & scalar & Soft count of calibration samples from class \(c\) assigned to cluster \(k\). \\
\(S_k\) & scalar & Effective support of cluster \(k\), i.e., total soft mass assigned to cluster \(k\). \\
\(\pi \in \Delta^{C-1}\) & vector & Global prior used for smoothing cluster label distributions. \\
\(\beta\) & scalar & Dirichlet smoothing strength toward \(\pi\). \\
\(\hat p_k(c)\) & scalar & Smoothed empirical probability of class \(c\) in cluster \(k\). \\
\(q^{\mathrm{local}}_x(c)\) & scalar & Local cluster-frequency mixture probability for class \(c\) at test point \(x\). \\
\(\pi_x(c)\) & scalar & Fallback prior for class \(c\) at \(x\), e.g., the base model probability. \\
\(\gamma\) & scalar & Exponent controlling the effect of assignment concentration in the reliability score. \\
\(\beta_{\mathrm{sup}}\) & scalar & Support-shrinkage parameter controlling the effect of neighborhood support. \\
\(r(x)\in[0,1]\) & scalar & Reliability weight used to interpolate between local cluster frequencies and the fallback prior. \\
\(q_x(c)\) & scalar & Final CFCP probability assigned to class \(c\) for input \(x\). \\
\(s(x,c)\) & scalar & Nonconformity score of candidate class \(c\) for input \(x\). \\
\(\alpha\) & scalar & Target miscoverage level. \\
\(\hat q_{1-\alpha}\) & scalar & Split-conformal quantile calibrated on \(\mathcal{D}_q\). \\
\(\hat\Gamma(x)\) & set & Conformal prediction set returned for input \(x\). \\
\bottomrule
\end{tabular}
\end{table}

\subsection{CFCP score construction details}

CFCP builds a sample-specific class-probability vector from local cluster-frequency information in a learned representation space. Let $\phi(x)\in\mathbb{R}^d$ denote the representation of input $x$, and let all embeddings and centroids be $\ell_2$-normalized. After fitting spherical $k$-means on the training representations, we obtain centroids $\mu_1,\dots,\mu_K$.

For a point $x$ with normalized embedding $z$, we retrieve its top-$m$ nearest centroids $c_1(x),\dots,c_m(x)$ with cosine similarities $\delta_1(x),\dots,\delta_m(x)$. These similarities are converted into soft assignment weights
\begin{equation}
w_j(x)=\frac{\exp(\delta_j(x)/\tau)}{\sum_{\ell=1}^m \exp(\delta_\ell(x)/\tau)}, \qquad j=1,\dots,m,
\end{equation}
where $\tau>0$ controls the sharpness of the local assignment.  Lower values yield more local assignments.

Let $D_{\mathrm{stat}}$ denote the calibration subset used to estimate local label statistics. Each calibration point contributes fractionally to its neighboring clusters. The resulting soft cluster-label counts and effective supports are
\begin{align}
N_{kc} &= \sum_{(x_i,y_i)\in D_{\mathrm{stat}}}\sum_{j=1}^m
w_{ij}\,\mathbf{1}\{k_j(x_i)=k,\; y_i=c\},\\
S_k &= \sum_{(x_i,y_i)\in D_{\mathrm{stat}}}\sum_{j=1}^m
w_{ij}\,\mathbf{1}\{k_j(x_i)=k\}.
\end{align}

We then estimate a smoothed cluster-level label distribution
\begin{equation}
\hat p_k(c)=\frac{N_{kc}+\beta \pi_c}{S_k+\beta}, \qquad c=1,\dots,C,
\end{equation}
where $\beta>0$ is a smoothing parameter and $\pi\in\Delta^{C-1}$ is a global prior. In our implementation, $\pi$ can be taken either as the empirical label prior or as the mean of the base model's class-probability vectors on $D_{\mathrm{stat}}$.

For a test point $x$, the local cluster-frequency mixture is
\begin{equation}
q_x^{\mathrm{local}}(c)=\sum_{j=1}^m w_j(x)\hat p_{k_j(x)}(c), \qquad c=1,\dots,C.
\end{equation}
This vector serves as the local probabilistic summary of the representation-space neighborhood of $x$.

\subsection{Reliability shrinkage details}

The local mixture $q_x^{\mathrm{local}}$ may be unstable when a point lies between several centroids or when its neighboring clusters have little support in calibration. CFCP therefore applies reliability-aware shrinkage toward a fallback prior $\pi_x$.

We first define assignment concentration by
\begin{equation}
\mathrm{conf}(x)=\Big(\max_{1\le j\le m} w_j(x)\Big)^\gamma,
\end{equation}
where $\gamma>0$ controls how strongly the method rewards concentrated local assignments. Next, we define neighborhood support as the weighted average support of the neighboring clusters,
\begin{equation}
\mathrm{sup}(x)=\sum_{j=1}^m w_j(x) S_{k_j(x)},
\end{equation}
and convert it into a normalized support weight,
\begin{equation}
\mathrm{supw}(x)=\frac{\mathrm{sup}(x)}{\mathrm{sup}(x)+\beta_{\mathrm{sup}}},
\end{equation}
where $\beta_{\mathrm{sup}}>0$ regularizes the effect of local support.

The resulting reliability score is
\begin{equation}
r(x)=\mathrm{conf}(x)\,\mathrm{supw}(x)\in[0,1].
\end{equation}
Finally, CFCP interpolates between the local cluster-frequency mixture and the fallback prior:
\begin{equation}
q_x(c)=r(x)\,q_x^{\mathrm{local}}(c)+(1-r(x))\,\pi_x(c), \qquad c=1,\dots,C.
\end{equation}

Reliability shrinkage therefore serves two roles: it preserves locality when the cluster signal is strong and it prevents unstable local estimates from dominating when the neighborhood is ambiguous or weakly supported.

\subsection{LAC/APS/RAPS/SAPS instantiations}

After constructing $q_x$, CFCP applies a standard split-conformal set constructor. Let $D_q=\{(x_i,y_i)\}_{i=1}^{n_q}$ denote the calibration subset used only for quantile estimation. Given a nonconformity score $s(x,y)$ induced by the chosen constructor, we compute
\begin{equation}
S_i=s(x_i,y_i), \qquad i=1,\dots,n_q,
\end{equation}
and set the split-conformal threshold
\begin{equation}
\hat q_{1-\alpha}=S_{(k)}, \qquad k=\left\lceil (n_q+1)(1-\alpha)\right\rceil,
\end{equation}
where $S_{(k)}$ is the $k$-th order statistic.

\paragraph{LAC.}
For the LAC-style instantiation, the nonconformity score is
\begin{equation}
s_{\mathrm{LAC}}(x,y)=1-q_x(y).
\end{equation}
The resulting prediction set is
\begin{equation}
\hat\Gamma_{\mathrm{LAC}}(x)=\{c\in\{1,\dots,C\}: q_x(c)\ge 1-\hat q_{1-\alpha}\}.
\end{equation}
In implementation, we enforce non-emptiness by including the top-ranked label when needed.

\paragraph{APS.}
Let $\pi_x^{(1)},\dots,\pi_x^{(C)}$ denote the labels sorted in decreasing order of $q_x(c)$, and let
\begin{equation}
Q_x^{(j)} = q_x\!\left(\pi_x^{(j)}\right).
\end{equation}
If the true label $y$ appears at rank $r_x(y)$, the deterministic APS score is
\begin{equation}
s_{\mathrm{APS}}(x,y)=\sum_{j=1}^{r_x(y)} Q_x^{(j)}.
\end{equation}
Under randomized APS, the score becomes
\begin{equation}
s_{\mathrm{APS}}^{\mathrm{rand}}(x,y)
=
\sum_{j=1}^{r_x(y)-1} Q_x^{(j)} + U_x\,Q_x^{(r_x(y))},
\qquad U_x\sim \mathrm{Unif}(0,1).
\end{equation}
The APS prediction set is the smallest top-probability prefix whose cumulative score does not exceed $\hat q_{1-\alpha}$.

\paragraph{RAPS.}
RAPS augments APS with rank-dependent regularization. Let $\lambda\ge 0$ be the regularization strength and $k_{\mathrm{reg}}\in\mathbb{N}$ the un-penalized rank. The deterministic RAPS score is
\begin{equation}
s_{\mathrm{RAPS}}(x,y)
=
s_{\mathrm{APS}}(x,y)
+
\lambda \max\{0,\; r_x(y)-k_{\mathrm{reg}}\},
\end{equation}
and the randomized variant replaces $s_{\mathrm{APS}}(x,y)$ by $s_{\mathrm{APS}}^{\mathrm{rand}}(x,y)$. As in APS, the prediction set is the smallest prefix of the labels sorted by $q_x(c)$ whose adjusted cumulative score remains below $\hat q_{1-\alpha}$, with a top-1 fallback to avoid empty sets.

\paragraph{SAPS.}
Sorted Adaptive Prediction Sets (SAPS) is a rank-based alternative that depends on the ordering induced by $q_x$ rather than on its full cumulative mass. Let $\pi_x^{(1)},\dots,\pi_x^{(C)}$ be the labels sorted in decreasing order of $q_x(c)$, and let $p_{(1)}(x)=q_x(\pi_x^{(1)})$ denote the top-1 probability. For a label at rank $r_x(y)\in\{0,\dots,C-1\}$, SAPS assigns the score
\begin{equation}
s_{\mathrm{SAPS}}(x,y)=p_{(1)}(x)+w\,r_x(y),
\end{equation}
where $w>0$ is a constant rank increment. In the randomized variant, the score is
\begin{equation}
s_{\mathrm{SAPS}}^{\mathrm{rand}}(x,y)=p_{(1)}(x)+w\,r_x(y)-U_x w,
\qquad U_x\sim \mathrm{Unif}(0,1),
\end{equation}
with the rank-1 label using $p_{(1)}(x)-U_x p_{(1)}(x)$ instead. Thus SAPS constructs prediction sets by traversing the labels sorted by $q_x$ and adding a fixed penalty per lower rank, yielding a simpler rank-sensitive alternative to APS and RAPS. In CFCP, SAPS is applied in the same plug-in manner as the other set constructors, with the ranking determined by the local cluster-frequency probability vector $q_x$.

These instantiations make clear that CFCP is a plug-in conformal framework: the conformal score and thresholding rules are standard, and the only CFCP-specific component is the local construction of the probability vector $q_x$ from cluster-frequency information.

\subsection{Computational complexity}
\label{app:complexity}

CFCP is designed to scale to large embedding sets. Since the thresholding stage uses a standard split-conformal procedure such as LAC, APS, RAPS or SAPS, its computational overhead beyond ordinary conformal prediction is concentrated in the representation-space localization step, namely clustering and nearest-centroid retrieval. Training requires clustering only the training embeddings into $K$ centroids. At test time, inference compares each point against centroids rather than against all training examples, reducing the locality search to a top-$m$ nearest-centroid lookup. In practice, we implement both spherical $k$-means and nearest-centroid retrieval using FAISS (\cite{douze2025faiss}), and we process examples in chunks to control memory use.

Let $n_{\mathrm{train}}$ be the number of training embeddings, $d$ the embedding dimension, $K$ the number of clusters, $T$ the number of spherical $k$-means iterations and $n$ the number of points to score at inference or calibration time. A standard spherical $k$-means pass requires assigning each embedding to its nearest centroid and recomputing centroids, yielding time complexity $O(T\, n_{\mathrm{train}} K d)$. After clustering, FAISS nearest-centroid retrieval with a flat inner-product index over normalized centroids requires $O(Kd)$ work per query, hence $O(nKd)$ over $n$ queries, after which only the top-$m$ neighbors are retained. The remaining conformal thresholding and set-construction steps use the same order of computation as their standard split-conformal counterparts, applied to the probability vectors $q_x$.

Compared with point-wise neighborhood methods, this centroid-based design offers a favorable tradeoff between locality and scalability: the method remains representation-aware and locally adaptive, while avoiding full nearest-neighbor search over all embedded training samples.

\section{Validity proofs}
\label{app:validity}

This appendix formalizes the validity claims for CFCP in the disjoint-split regime. The proof strategy separates the standard split-conformal marginal guarantee from a stronger oracle local statement that is specific to randomized APS. The latter is inspired by the randomized oracle construction of \citet{romano2020classification}. Our theorem is a new specialization of that idea to the case where CFCP recovers the true conditional label law. The broader oracle perspective is also aligned with \citet{sadinle2019least}, while the need for additional assumptions in local validity is consistent with the conditional-validity literature \citep{vovk2012conditional,barber2023local}.

\subsection{Exact marginal validity}
\label{app:marginal-validity}
\subsubsection{Proof of Theorem~\ref{thm:cfcp-marginal}.}
\label{proof:cfcp-marginal}
Let $\widehat M$ denote all quantities fitted before quantile calibration, including the representation map, centroids, soft cluster statistics, smoothed cluster distributions, reliability terms, and any fallback prior used to form $q_x$. Assume that:
\begin{enumerate}
    \item $D_{\mathrm{train}}, D_{\mathrm{stat}}, D_q$, and the test point $(X_{n+1},Y_{n+1})$ are mutually disjoint, and
    $
    \{(X_i,Y_i): i\in D_q\}\cup\{(X_{n+1},Y_{n+1})\}
    $ are exchangeable.
    \item $\widehat M$ is constructed without using labels from $D_q$.
    \item Any auxiliary probabilities used in the fallback prior for points in $D_q$ and at test time are produced without using labels from $D_q$.
    \item If randomized APS or randomized RAPS is used, the auxiliary random variables are i.i.d. and independent of the data.
\end{enumerate}

Define the calibration scores
\[
S_i=s(X_i,Y_i), \qquad (X_i,Y_i)\in D_q,
\]
and let
\[
\widehat q_{1-\alpha}=S_{(k)}, \qquad
k=\left\lceil (|D_q|+1)(1-\alpha)\right\rceil ,
\]
where $S_{(k)}$ is the $k$th order statistic. Then, conditional on $\widehat M$,
\[
\Pr\!\left\{Y_{n+1}\in \widehat\Gamma(X_{n+1}) \mid \widehat M \right\}\ge 1-\alpha .
\]

\begin{proof}
Conditional on $\widehat M$, the score map $(x,y)\mapsto s(x,y)$ is fixed. By Assumptions 1-4, the calibration scores
$
\{S_i:(X_i,Y_i)\in D_q\}
$
and the test score
$
S_{n+1}=s(X_{n+1},Y_{n+1})
$
are exchangeable conditional on $\widehat M$. Standard split conformal rank arguments therefore imply
\[
\Pr\!\left\{S_{n+1}\le \widehat q_{1-\alpha}\mid \widehat M\right\}\ge 1-\alpha .
\]
By construction, \(S_{n+1}\le \hat q_{1-\alpha}\) implies
\(Y_{n+1}\in\hat{\Gamma}(X_{n+1})\). Any non-empty-set fallback can only enlarge
the prediction set and therefore preserves the coverage lower bound. Hence
\[
\Pr\!\left\{Y_{n+1}\in \widehat\Gamma(X_{n+1}) \mid \widehat M \right\}\ge 1-\alpha .
\]
\end{proof}
\subsubsection{Proof of Corollary \ref{cor:cfcp-marginal}}
Under the assumptions of Theorem~\ref{thm:cfcp-marginal},
\[
\Pr\!\left\{Y_{n+1}\in \widehat\Gamma(X_{n+1})\right\}\ge 1-\alpha .
\]
\begin{proof}
Take expectations in Theorem~\ref{thm:cfcp-marginal} with respect to $\widehat M$.
\end{proof}

\subsection{Proof of Corollary \ref{cor:cfcp-oracle-local-main}}
\label{app:cfcp-oracle-local-main}
The next result is specific to randomized APS. Its intuition comes from the oracle randomized classifier of \citet{romano2020classification}: if the predictive probability vector equals the true conditional label law, then the randomized cumulative-probability score acts as a discrete probability integral transform.

\begin{lemma}[Conditional uniformity of the randomized APS score]
\label{lem:aps-uniform}
Fix $x\in\mathcal X$, and suppose
$
q_x(c)=\eta_c(x):=\Pr(Y=c\mid X=x)
$
for all $c\in\{1,\dots,C\}$. Let
$
\pi_x^{(1)},\dots,\pi_x^{(C)}
$
be the labels sorted in decreasing order of $q_x(c)$, and let
$
Q_x^{(j)}=q_x(\pi_x^{(j)}).
$
If $Y\sim \eta(\cdot\mid x)$ and $U\sim \mathrm{Unif}(0,1)$ is independent, define
\[
S^{\mathrm{randAPS}}(x,Y,U)
=
\sum_{j=1}^{r_x(Y)-1} Q_x^{(j)} + U\, Q_x^{(r_x(Y))},
\]
where $r_x(Y)$ is the rank of $Y$ in the ordering induced by $q_x$. Then
\[
S^{\mathrm{randAPS}}(x,Y,U)\mid X=x \sim \mathrm{Unif}(0,1).
\]
\end{lemma}

\begin{proof}
For a fixed $x$, the sorted probabilities partition $[0,1]$ into intervals. Let:
\[
B_j(x)=\sum_{\ell<j} Q_x^{(\ell)}, \qquad j=1,\ldots,C,
\]
and define
\[
I_j(x)=[B_j(x), B_j(x)+Q_x^{(j)}), \qquad j<C,
\]
with the last interval closed on the right. Then the intervals partition \([0,1]\) and \(|I_j(x)|=Q_x^{(j)}\).

Under $Y\sim \eta(\cdot\mid x)=q_x$, label
$\pi_x^{(j)}$ is selected with probability $Q_x^{(j)}$. Conditional on selecting that label, the randomized score
$
S^{\mathrm{randAPS}}(x,Y,U)
$
is uniform on $I_j(x)$. Therefore, conditional on $X=x$, the random variable $S^{\mathrm{randAPS}}(x,Y,U)$ is a mixture of uniforms on the intervals $I_j(x)$ with weights equal to the interval lengths $|I_j(x)|=Q_x^{(j)}$, and hence
\[
S^{\mathrm{randAPS}}(x,Y,U)\mid X=x \sim \mathrm{Unif}(0,1).
\]
\end{proof}

\begin{theorem}[Oracle local validity of disjoint-split CFCP with randomized APS]
\label{thm:oracle-local}
Assume the setting of Theorem~\ref{thm:cfcp-marginal}. In addition, assume:
\begin{enumerate}
    \item \textbf{Oracle recovery:}
    \[
    q_x(c)=\eta_c(x):=\Pr(Y=c\mid X=x), \qquad \forall x,\ c.
    \]
    \item \textbf{Randomized APS score:}
    CFCP uses randomized APS, i.e.
    \[
    S^{\mathrm{randAPS}}(x,y;u)
    =
    \sum_{j=1}^{r_x(y)-1} Q_x^{(j)} + u\,Q_x^{(r_x(y))},
    \qquad u\sim \mathrm{Unif}(0,1),
    \]
    with independent randomization and a fixed tie-breaking rule.
\end{enumerate}

Let $\widehat q_{1-\alpha}$ be the split-conformal quantile computed on $D_q$, and let $\widehat\Gamma$ be the resulting randomized APS prediction set. Then for every measurable region $A\subseteq\mathcal X$ defined without using labels from $D_q$,
\[
\Pr\!\left\{Y_{n+1}\in \widehat\Gamma(X_{n+1}) \mid X_{n+1}\in A\right\}\ge 1-\alpha.
\]
\end{theorem}

\begin{proof}
By Lemma~\ref{lem:aps-uniform}, for every fixed $x$,
\[
S_{n+1}:=S^{\mathrm{randAPS}}(X_{n+1},Y_{n+1};U_{n+1}) \mid X_{n+1}=x
\sim \mathrm{Unif}(0,1).
\]
Hence for every measurable region $A$,
\[
S_{n+1}\mid \{X_{n+1}\in A\}\sim \mathrm{Unif}(0,1).
\]
The same is true for the calibration scores in $D_q$, since the oracle assumption holds for all calibration points as well. Because $A$ is defined independently of the labels in $D_q$, conditioning on $X_{n+1}\in A$ does not disturb the exchangeability of the calibration scores and the test score conditional on $\widehat M$. Therefore, the split-conformal rank argument applies conditionally on $\{X_{n+1}\in A\}$:
\[
\Pr\!\left\{S_{n+1}\le \widehat q_{1-\alpha}\mid X_{n+1}\in A\right\}\ge 1-\alpha.
\]
Finally, by the definition of randomized APS,
\[
Y_{n+1}\in \widehat\Gamma(X_{n+1})
\quad \Longleftrightarrow \quad
S_{n+1}\le \widehat q_{1-\alpha},
\]
which proves the claim.
\end{proof}

\begin{corollary}[Piecewise-constant cluster oracle]
\label{cor:cluster-oracle}
Assume there exists a measurable partition
\[
\mathcal X = A_1\cup\cdots\cup A_K
\]
such that for each $k$ there is a probability vector $p_k\in\Delta^{C-1}$ satisfying
\[
\Pr(Y=c\mid X=x)=p_k(c), \qquad \forall x\in A_k,\ \forall c.
\]
Assume further that CFCP recovers this oracle structure in the sense that for every $x\in A_k$,
\[
q_x=p_k.
\]
Then the conclusion of Theorem~\ref{thm:oracle-local} holds on every cluster $A_k$:
\[
\Pr\!\left\{Y_{n+1}\in \widehat\Gamma(X_{n+1}) \mid X_{n+1}\in A_k\right\}\ge 1-\alpha,
\qquad k=1,\dots,K.
\]
\end{corollary}

\begin{proof}
Apply Theorem~\ref{thm:oracle-local} with $A=A_k$.
\end{proof}

\paragraph{Remark: (Why this theorem is specific to randomized APS)}
\label{rem:local-limitation}
Theorem~\ref{thm:oracle-local} relies on the conditional uniformity in Lemma~\ref{lem:aps-uniform}. This property is natural for randomized APS, but it does not hold in general for the LAC score $1-q_x(y)$, deterministic APS, RAPS or SAPS. Therefore we do not claim exact local validity for the full CFCP score-family. Outside randomized APS, local validity should be understood through approximate score-homogenization statements.

\subsection{Approximate local validity}
\label{app:approx-local}
The next proposition formalizes the general intuition of CFCP: if the local score distribution within a region $A$ is close to the global score distribution used for calibration, then the local coverage error in $A$ is small.

\begin{proposition}[Approximate local validity from score homogenization]
\label{prop:approx-local}
Assume the setting of Theorem~\ref{thm:cfcp-marginal}, and let
\[
S=s(X,Y;U)
\]
denote the nonconformity score of a fresh point under any supported CFCP set constructor.
Let
\[
F(t)=\Pr(S\le t), \qquad
F_A(t)=\Pr(S\le t\mid X\in A)
\]
be the global and local score distribution functions, respectively, for a measurable region $A\subseteq\mathcal X$. Let
\[
q^*_{1-\alpha}=\inf\{t: F(t)\ge 1-\alpha\}.
\]
If
\[
\sup_t |F_A(t)-F(t)|\le \varepsilon_A
\]
for some $\varepsilon_A\ge 0$, then the population-thresholded prediction set
\[
\Gamma^*(x)=\{c: s(x,c;U)\le q^*_{1-\alpha}\}
\]
satisfies
\[
\Pr\!\left\{Y_{n+1}\in \Gamma^*(X_{n+1}) \mid X_{n+1}\in A\right\}
\ge 1-\alpha-\varepsilon_A.
\]
\end{proposition}

\begin{proof}
By definition of $q^*_{1-\alpha}$,
\[
F(q^*_{1-\alpha})\ge 1-\alpha.
\]
Therefore
\[
F_A(q^*_{1-\alpha})
\ge
F(q^*_{1-\alpha})-\sup_t |F_A(t)-F(t)|
\ge
1-\alpha-\varepsilon_A.
\]
Since
\[
\{Y_{n+1}\in \Gamma^*(X_{n+1})\}
=
\{S_{n+1}\le q^*_{1-\alpha}\},
\]
the result follows.
\end{proof}

\paragraph{Remark (Connection to the finite-sample CFCP procedure)}
\label{rem:empirical-approx}
Proposition~\ref{prop:approx-local} is stated at the population quantile $q^*_{1-\alpha}$ to isolate the local-homogenization effect. In the actual split-conformal procedure, $q^*_{1-\alpha}$ is replaced by the empirical order-statistic threshold $\widehat q_{1-\alpha}$. The resulting set inherits the same qualitative conclusion, up to the usual finite-sample rank-calibration conservatism already captured by Theorem~\ref{thm:cfcp-marginal}.

\section{Experiments and results}
\subsection{Experimental setup}
\label{app:experimental_setup}

\subsubsection{Data sets and models}
\begin{itemize}
\item \textbf{CIFAR-100} (\cite{Krizhevsky09learningmultiple}): we followed the training procedure from \cite{ding2023manyclasses}.
\begin{itemize}
    \item Splits: The datasets include 50,000 train images and 10,000 test images equally distributed over 10/100 classes. We split the training data to 60\% and 40\% validation. The test data was split to five repeated holdouts with fixed seeds to 75\% calibration, and 25\% testing.
    \item Model: PyTorch (\cite{paszke2019pytorch}) ResNet-50 (\cite{he2016deep}),  initialized to the IMAGENET1K\_V2 (\cite{deng2009imagenet}) weights and then fine-tuned all parameters for 30 epochs. Validation accuracy is 60\%.
\end{itemize}
\item \textbf{ImageNet} (\cite{deng2009imagenet}): we followed the training procedure from \cite{ding2023manyclasses}. 
\begin{itemize}
    \item Splits: We sampled random 100,000 images of the $\approx$1M training set for training and another five repeated holdouts with fixed seeds of 100,000 images for conformal prediction calibration and validation. We used the 50,000 ImageNet's validation set for testing. In addition, we used ImageNet-V2's test set using the classification and conformal prediction models of ImageNet's experiment.
    \item Model: Based on a SimCLR-v2 model (\cite{chen2020big}), which was trained on the ImageNet training set without labels, to extract feature vectors of length 6,144 for all images in the ImageNet training set, we trained a classifier using one fully connected prediction head. Validation accuracy is 78\%.
\end{itemize}
\item \textbf{WOS-46985}: We used the keywords and abstract texts to train a classification model on the dataset.
\begin{itemize}
    \item Splits: 70\% of the data were randomly sampled for training. The remaining 30\% were split to five repeated holdouts with fixed seeds to 65\% conformal prediction calibration and 35\% testing.
    \item Model: We used DistilBERT (\cite{sanh2019distilbert}) classifier fine-tuned with num\_train\_epochs=5, learning\_rate=2e-5, weight\_decay=0.01, warmup\_ratio=0.1. Validation accuracy is 88.4\%.
\end{itemize}
\item \textbf{Baseline methods}: We compared our CFCP method to Split, ICP, CCP, NCP and RC3P by evaluating each on these score methods: LAC, APS, RAPS and SAPS. For the Split, ICP and CCP, baseline methods and all scores implementation we used the TorchCP framework (\cite{huang2024torchcp}). For the NCP and RC3P baselines we repeated the implementation that was shared by the authors in the original papers.
    \item \textbf{Method parameters}: 
    \begin{itemize}
        \item Score: We used equal parameters for scoring on both CFCP and baseline methods. APS/RAPS/SAPS: randomized=True, RAPS: penalty=0.1, kreg=5 for CIFAR-100 and WOS-46985, 20 for ImageNet, SAPS: weight=0.2.
        \item Prediction set construction: We used temperature=1 for all methods.
        \item CCP: ratio\_clustering="auto", num\_clusters="auto".
        \item NCP:  localizer tuned using "logits" from the authors' implementation and "features" as mentioned in the original paper (we used CNN activations and Transformer hidden states), use\_platt\_scaling=True, k tuned using grid \{0.05, 0.10, 0.20, 0.40\}, lambda\_L tuned using grid \{10, 50, 100, 500, 1000, 5000\}, hyper\_fraction=0.3.
        \item RC3P: variant = 'mix', mix\_grid\_size=50, hyper\_fraction=0, coverage\_slack=0.
        \item CFCP locality-specific hyperparameters were tuned per dataset: n\_clusters ($K$), n\_neighbors ($m$), cluster assignment $\tau$, Dirichlet smoothing ($\beta$), reliability gamma ($\gamma$), reliability beta support ($\beta_{\mathrm{sup}}$).  These hyperparameters were selected on an internal tuning split from the calibration pool. 
\end{itemize}
\item \textbf{Calibration sub-splits}: We tuned locality-specific hyperparameters per dataset: \(n_{\text{clusters}}(K)\), \(n_{\text{neighbors}}(m)\), cluster-assignment temperature \(\tau\), Dirichlet smoothing \(\beta\), reliability gamma \(\gamma\), and reliability beta support \(\beta_{\text{sup}}\). For each outer data split, we partitioned the non-test conformal pool into \(D_{\text{stat}}, D_{\text{tune}}, D_q\) with proportions (0.6/0.2/0.2). \(D_{\text{stat}}\) was used to estimate the local cluster-frequency model, \(D_{\text{tune}}\) was used only for hyperparameter selection via an internal split into quantile-calibration and evaluation subsets, and \(D_q\) was reserved exclusively for the final conformal quantile. After hyperparameter selection, we refit the CFCP local model on \(D_{\text{stat}} \cup D_{\text{tune}}\) and calibrated the final conformal threshold on the untouched \(D_q\). Baseline methods were calibrated on the same calibration pool and all methods were evaluated on the same untouched test set.
\end{itemize}

\subsubsection{Compute}
All experiments were performed on a Macbook Pro M4 Max 36GB RAM. A full run of all reported CFCP experiments required approximately 3-4 hours. Individual dataset runs required approximately 30-60 minutes depending on the dataset.

\subsection{Additional results}
\label{app:results}
\begin{longtable}{@{} l l r r r r @{}}
\caption{Full results.  Each metric column reports mean\,$\pm$\,95\,\%\,CI over five repeated splits.  \textbf{Bold} marks the best class coverage within each (Dataset / Score) group.  }\label{tab:full_results}\\
\toprule
\textbf{Score} & \textbf{Method} & \textbf{Class Cov.\,($\uparrow$)} & \textbf{Set Size\,($\downarrow$)} & \textbf{WUC\,($\downarrow$)} & \textbf{MaxCE\,($\downarrow$)} \\
\multicolumn{2}{@{}l}{\textit{(metrics: mean\,$\pm$\,95\,\%\,CI)}} & & & & \\
\midrule
\endfirsthead
\toprule
\textbf{Score} & \textbf{Method} & \textbf{Class Cov.\,($\uparrow$)} & \textbf{Set Size\,($\downarrow$)} & \textbf{WUC\,($\downarrow$)} & \textbf{MaxCE\,($\downarrow$)} \\
\multicolumn{2}{@{}l}{\textit{(metrics: mean\,$\pm$\,95\,\%\,CI)}} & & & & \\
\midrule
\endhead
\midrule
\multicolumn{6}{r@{}}{\footnotesize\textit{(continued on next page)}}\\
\endfoot
\bottomrule
\endlastfoot
\multicolumn{6}{@{}l@{}}{\cellcolor{gray!15}\small \textbf{CIFAR-100}} \\
\addlinespace[1pt]
\midrule\multirow{6}{*}{\textit{LAC}} & CFCP & $\mathbf{0.592 \pm 0.057}$ & $7.9 \pm 0.4$ & $0.030 \pm 0.009$ & $0.196 \pm 0.027$ \\
 & RC3P & $0.568 \pm 0.042$ & $11.2 \pm 0.5$ & $0.026 \pm 0.005$ & $0.164 \pm 0.027$ \\
 & ICP & $0.568 \pm 0.042$ & $11.2 \pm 0.5$ & $0.026 \pm 0.005$ & $0.164 \pm 0.027$ \\
 & Split & $0.516 \pm 0.034$ & $8.0 \pm 0.2$ & $0.033 \pm 0.006$ & $0.228 \pm 0.042$ \\
 & NCP & $0.566 \pm 0.056$ & $8.9 \pm 0.5$ & $0.023 \pm 0.005$ & $0.164 \pm 0.044$ \\
 & CCP & $0.534 \pm 0.069$ & $9.7 \pm 0.5$ & $0.027 \pm 0.007$ & $0.204 \pm 0.027$ \\
\addlinespace[2pt]
\midrule\multirow{6}{*}{\textit{APS}} & CFCP & $\mathbf{0.590 \pm 0.039}$ & $8.3 \pm 0.5$ & $0.025 \pm 0.009$ & $0.220 \pm 0.070$ \\
 & RC3P & $0.582 \pm 0.032$ & $11.3 \pm 0.1$ & $0.025 \pm 0.002$ & $0.196 \pm 0.057$ \\
 & ICP & $0.560 \pm 0.046$ & $14.6 \pm 0.6$ & $0.026 \pm 0.005$ & $0.188 \pm 0.022$ \\
 & Split & $0.516 \pm 0.045$ & $11.1 \pm 0.3$ & $0.032 \pm 0.006$ & $0.212 \pm 0.022$ \\
 & NCP & $0.572 \pm 0.039$ & $12.0 \pm 0.5$ & $0.024 \pm 0.005$ & $0.172 \pm 0.022$ \\
 & CCP & $0.530 \pm 0.066$ & $12.2 \pm 1.0$ & $0.029 \pm 0.009$ & $0.188 \pm 0.022$ \\
\addlinespace[2pt]
\midrule\multirow{6}{*}{\textit{RAPS}} & CFCP & $\mathbf{0.598 \pm 0.034}$ & $9.1 \pm 0.6$ & $0.031 \pm 0.007$ & $0.212 \pm 0.042$ \\
 & RC3P & $0.564 \pm 0.037$ & $11.6 \pm 0.2$ & $0.025 \pm 0.003$ & $0.172 \pm 0.042$ \\
 & ICP & $0.570 \pm 0.041$ & $11.6 \pm 0.2$ & $0.025 \pm 0.002$ & $0.172 \pm 0.042$ \\
 & Split & $0.556 \pm 0.034$ & $8.3 \pm 0.1$ & $0.036 \pm 0.004$ & $0.252 \pm 0.074$ \\
 & NCP & $0.562 \pm 0.054$ & $7.9 \pm 0.8$ & $0.027 \pm 0.008$ & $0.204 \pm 0.075$ \\
 & CCP & $0.580 \pm 0.044$ & $10.6 \pm 0.8$ & $0.027 \pm 0.005$ & $0.196 \pm 0.027$ \\
\addlinespace[2pt]
\midrule\multirow{6}{*}{\textit{SAPS}} & CFCP & $\mathbf{0.586 \pm 0.031}$ & $8.8 \pm 0.5$ & $0.033 \pm 0.008$ & $0.212 \pm 0.042$ \\
 & RC3P & $0.548 \pm 0.048$ & $10.9 \pm 0.2$ & $0.027 \pm 0.003$ & $0.188 \pm 0.042$ \\
 & ICP & $0.582 \pm 0.051$ & $11.5 \pm 0.2$ & $0.024 \pm 0.003$ & $0.188 \pm 0.042$ \\
 & Split & $0.558 \pm 0.024$ & $8.2 \pm 0.1$ & $0.036 \pm 0.005$ & $0.244 \pm 0.057$ \\
 & NCP & $0.566 \pm 0.078$ & $7.7 \pm 0.8$ & $0.027 \pm 0.011$ & $0.196 \pm 0.057$ \\
 & CCP & $0.568 \pm 0.051$ & $10.3 \pm 0.5$ & $0.028 \pm 0.002$ & $0.204 \pm 0.027$ \\
\addlinespace[4pt]
\multicolumn{6}{@{}l@{}}{\cellcolor{gray!15}\small \textbf{ImageNet}} \\
\addlinespace[1pt]
\midrule\multirow{6}{*}{\textit{LAC}} & CFCP & $\mathbf{0.508 \pm 0.004}$ & $1.9 \pm 0.0$ & $0.058 \pm 0.001$ & $0.520 \pm 0.000$ \\
 & RC3P & $0.425 \pm 0.009$ & $3.5 \pm 0.1$ & $0.049 \pm 0.001$ & $0.516 \pm 0.122$ \\
 & ICP & $0.493 \pm 0.008$ & $3.9 \pm 0.1$ & $0.038 \pm 0.002$ & $0.484 \pm 0.123$ \\
 & Split & $0.499 \pm 0.005$ & $1.8 \pm 0.0$ & $0.060 \pm 0.001$ & $0.540 \pm 0.000$ \\
 & NCP & $0.494 \pm 0.006$ & $1.9 \pm 0.0$ & $0.063 \pm 0.002$ & $0.500 \pm 0.000$ \\
 & CCP & $0.494 \pm 0.026$ & $3.3 \pm 0.3$ & $0.040 \pm 0.003$ & $0.556 \pm 0.130$ \\
\addlinespace[2pt]
\midrule\multirow{6}{*}{\textit{APS}} & CFCP & $\mathbf{0.672 \pm 0.002}$ & $15.1 \pm 0.2$ & $0.030 \pm 0.000$ & $0.380 \pm 0.000$ \\
 & RC3P & $0.654 \pm 0.008$ & $8.4 \pm 0.5$ & $0.023 \pm 0.001$ & $0.424 \pm 0.062$ \\
 & ICP & $0.658 \pm 0.026$ & $42.8 \pm 0.4$ & $0.019 \pm 0.001$ & $0.240 \pm 0.053$ \\
 & Split & $0.655 \pm 0.008$ & $35.2 \pm 0.3$ & $0.023 \pm 0.001$ & $0.392 \pm 0.028$ \\
 & NCP & $0.667 \pm 0.009$ & $15.2 \pm 0.4$ & $0.033 \pm 0.001$ & $0.448 \pm 0.014$ \\
 & CCP & $0.646 \pm 0.010$ & $38.0 \pm 0.9$ & $0.020 \pm 0.001$ & $0.276 \pm 0.027$ \\
\addlinespace[2pt]
\midrule\multirow{6}{*}{\textit{RAPS}} & CFCP & $\mathbf{0.589 \pm 0.002}$ & $3.6 \pm 0.0$ & $0.042 \pm 0.000$ & $0.440 \pm 0.000$ \\
 & RC3P & $0.546 \pm 0.008$ & $5.9 \pm 0.3$ & $0.033 \pm 0.001$ & $0.444 \pm 0.044$ \\
 & ICP & $0.546 \pm 0.007$ & $5.5 \pm 0.3$ & $0.030 \pm 0.001$ & $0.352 \pm 0.076$ \\
 & Split & $0.553 \pm 0.017$ & $3.3 \pm 0.0$ & $0.039 \pm 0.001$ & $0.404 \pm 0.021$ \\
 & NCP & $0.565 \pm 0.005$ & $2.8 \pm 0.0$ & $0.047 \pm 0.000$ & $0.500 \pm 0.030$ \\
 & CCP & $0.536 \pm 0.035$ & $8.5 \pm 2.4$ & $0.034 \pm 0.003$ & $0.400 \pm 0.058$ \\
\addlinespace[2pt]
\midrule\multirow{6}{*}{\textit{SAPS}} & CFCP & $\mathbf{0.571 \pm 0.004}$ & $2.6 \pm 0.0$ & $0.046 \pm 0.001$ & $0.480 \pm 0.000$ \\
 & RC3P & $0.556 \pm 0.010$ & $6.0 \pm 0.3$ & $0.032 \pm 0.001$ & $0.444 \pm 0.044$ \\
 & ICP & $0.533 \pm 0.021$ & $5.1 \pm 0.2$ & $0.031 \pm 0.001$ & $0.440 \pm 0.053$ \\
 & Split & $0.529 \pm 0.007$ & $2.2 \pm 0.0$ & $0.055 \pm 0.000$ & $0.476 \pm 0.032$ \\
 & NCP & $0.559 \pm 0.007$ & $2.4 \pm 0.0$ & $0.049 \pm 0.001$ & $0.448 \pm 0.014$ \\
 & CCP & $0.549 \pm 0.020$ & $6.8 \pm 0.9$ & $0.033 \pm 0.001$ & $0.436 \pm 0.037$ \\
\addlinespace[4pt]
\multicolumn{6}{@{}l@{}}{\cellcolor{gray!15}\small \textbf{ImageNet-V2}} \\
\addlinespace[1pt]
\midrule\multirow{6}{*}{\textit{LAC}} & CFCP & $0.446 \pm 0.002$ & $2.5 \pm 0.0$ & $0.128 \pm 0.001$ & $0.900 \pm 0.000$ \\
 & RC3P & $0.424 \pm 0.014$ & $6.9 \pm 0.5$ & $0.125 \pm 0.003$ & $0.720 \pm 0.056$ \\
 & ICP & $0.388 \pm 0.008$ & $5.0 \pm 0.1$ & $0.132 \pm 0.002$ & $0.740 \pm 0.068$ \\
 & Split & $\mathbf{0.452 \pm 0.002}$ & $2.4 \pm 0.0$ & $0.128 \pm 0.001$ & $0.900 \pm 0.000$ \\
 & NCP & $0.445 \pm 0.009$ & $2.8 \pm 0.1$ & $0.130 \pm 0.002$ & $0.800 \pm 0.000$ \\
 & CCP & $0.386 \pm 0.019$ & $4.3 \pm 0.2$ & $0.133 \pm 0.006$ & $0.800 \pm 0.124$ \\
\addlinespace[2pt]
\midrule\multirow{6}{*}{\textit{APS}} & CFCP & $\mathbf{0.714 \pm 0.002}$ & $31.2 \pm 0.4$ & $0.046 \pm 0.000$ & $0.500 \pm 0.000$ \\
 & RC3P & $0.712 \pm 0.004$ & $28.1 \pm 1.1$ & $0.047 \pm 0.001$ & $0.560 \pm 0.068$ \\
 & ICP & $0.708 \pm 0.013$ & $72.0 \pm 0.5$ & $0.042 \pm 0.002$ & $0.440 \pm 0.068$ \\
 & Split & $0.710 \pm 0.011$ & $64.7 \pm 0.3$ & $0.043 \pm 0.001$ & $0.520 \pm 0.056$ \\
 & NCP & $0.704 \pm 0.008$ & $39.8 \pm 1.1$ & $0.048 \pm 0.002$ & $0.600 \pm 0.000$ \\
 & CCP & $0.707 \pm 0.019$ & $69.2 \pm 4.1$ & $0.042 \pm 0.003$ & $0.460 \pm 0.068$ \\
\addlinespace[2pt]
\midrule\multirow{6}{*}{\textit{RAPS}} & CFCP & $\mathbf{0.552 \pm 0.004}$ & $4.8 \pm 0.0$ & $0.092 \pm 0.001$ & $0.700 \pm 0.000$ \\
 & RC3P & $0.544 \pm 0.012$ & $10.2 \pm 0.5$ & $0.087 \pm 0.002$ & $0.720 \pm 0.056$ \\
 & ICP & $0.523 \pm 0.007$ & $6.4 \pm 0.1$ & $0.088 \pm 0.001$ & $0.680 \pm 0.056$ \\
 & Split & $0.525 \pm 0.007$ & $4.3 \pm 0.0$ & $0.095 \pm 0.001$ & $0.700 \pm 0.000$ \\
 & NCP & $0.541 \pm 0.006$ & $4.2 \pm 0.0$ & $0.094 \pm 0.001$ & $0.700 \pm 0.000$ \\
 & CCP & $0.529 \pm 0.012$ & $8.8 \pm 1.1$ & $0.087 \pm 0.001$ & $0.680 \pm 0.056$ \\
\addlinespace[2pt]
\midrule\multirow{6}{*}{\textit{SAPS}} & CFCP & $\mathbf{0.465 \pm 0.002}$ & $3.0 \pm 0.0$ & $0.122 \pm 0.001$ & $0.900 \pm 0.000$ \\
 & RC3P & $0.440 \pm 0.008$ & $6.0 \pm 0.3$ & $0.116 \pm 0.002$ & $0.740 \pm 0.068$ \\
 & ICP & $0.463 \pm 0.010$ & $5.5 \pm 0.2$ & $0.104 \pm 0.002$ & $0.720 \pm 0.056$ \\
 & Split & $0.458 \pm 0.003$ & $2.8 \pm 0.0$ & $0.123 \pm 0.001$ & $0.820 \pm 0.056$ \\
 & NCP & $0.446 \pm 0.004$ & $2.8 \pm 0.0$ & $0.125 \pm 0.001$ & $0.800 \pm 0.000$ \\
 & CCP & $0.455 \pm 0.006$ & $7.4 \pm 1.5$ & $0.108 \pm 0.003$ & $0.740 \pm 0.068$ \\
\addlinespace[4pt]
\multicolumn{6}{@{}l@{}}{\cellcolor{gray!15}\small \textbf{WOS-46985}} \\
\addlinespace[1pt]
\midrule\multirow{6}{*}{\textit{LAC}} & CFCP & $\mathbf{0.645 \pm 0.031}$ & $2.5 \pm 0.3$ & $0.029 \pm 0.005$ & $0.710 \pm 0.135$ \\
 & RC3P & $0.634 \pm 0.054$ & $3.8 \pm 0.6$ & $0.020 \pm 0.002$ & $0.312 \pm 0.082$ \\
 & ICP & $0.621 \pm 0.064$ & $4.1 \pm 0.4$ & $0.021 \pm 0.002$ & $0.312 \pm 0.082$ \\
 & Split & $0.622 \pm 0.017$ & $1.2 \pm 0.0$ & $0.035 \pm 0.003$ & $0.510 \pm 0.188$ \\
 & NCP & $0.619 \pm 0.031$ & $1.3 \pm 0.1$ & $0.027 \pm 0.003$ & $0.810 \pm 0.155$ \\
 & CCP & $0.601 \pm 0.075$ & $5.4 \pm 1.1$ & $0.027 \pm 0.002$ & $0.324 \pm 0.076$ \\
\addlinespace[2pt]
\midrule\multirow{6}{*}{\textit{APS}} & CFCP & $\mathbf{0.633 \pm 0.018}$ & $1.5 \pm 0.0$ & $0.038 \pm 0.003$ & $0.837 \pm 0.107$ \\
 & RC3P & $0.587 \pm 0.052$ & $3.8 \pm 0.5$ & $0.023 \pm 0.002$ & $0.262 \pm 0.113$ \\
 & ICP & $0.606 \pm 0.053$ & $6.2 \pm 0.5$ & $0.020 \pm 0.004$ & $0.278 \pm 0.093$ \\
 & Split & $0.546 \pm 0.034$ & $3.4 \pm 0.2$ & $0.023 \pm 0.004$ & $0.300 \pm 0.113$ \\
 & NCP & $0.609 \pm 0.020$ & $1.5 \pm 0.1$ & $0.036 \pm 0.001$ & $0.475 \pm 0.132$ \\
 & CCP & $0.570 \pm 0.086$ & $6.6 \pm 0.6$ & $0.022 \pm 0.004$ & $0.240 \pm 0.042$ \\
\addlinespace[2pt]
\midrule\multirow{6}{*}{\textit{RAPS}} & CFCP & $\mathbf{0.633 \pm 0.018}$ & $1.5 \pm 0.0$ & $0.038 \pm 0.003$ & $0.837 \pm 0.107$ \\
 & RC3P & $0.579 \pm 0.024$ & $4.0 \pm 0.4$ & $0.020 \pm 0.002$ & $0.261 \pm 0.117$ \\
 & ICP & $0.609 \pm 0.051$ & $5.9 \pm 0.7$ & $0.020 \pm 0.004$ & $0.243 \pm 0.117$ \\
 & Split & $0.560 \pm 0.057$ & $2.8 \pm 0.1$ & $0.021 \pm 0.005$ & $0.300 \pm 0.113$ \\
 & NCP & $0.600 \pm 0.019$ & $1.5 \pm 0.0$ & $0.036 \pm 0.002$ & $0.503 \pm 0.134$ \\
 & CCP & $0.587 \pm 0.047$ & $6.0 \pm 0.6$ & $0.021 \pm 0.005$ & $0.205 \pm 0.030$ \\
\addlinespace[2pt]
\midrule\multirow{6}{*}{\textit{SAPS}} & CFCP & $\mathbf{0.643 \pm 0.025}$ & $2.3 \pm 0.1$ & $0.048 \pm 0.004$ & $0.827 \pm 0.126$ \\
 & RC3P & $0.604 \pm 0.051$ & $3.7 \pm 0.3$ & $0.021 \pm 0.003$ & $0.270 \pm 0.101$ \\
 & ICP & $0.610 \pm 0.051$ & $5.8 \pm 0.6$ & $0.021 \pm 0.003$ & $0.257 \pm 0.046$ \\
 & Split & $0.590 \pm 0.046$ & $1.3 \pm 0.0$ & $0.029 \pm 0.002$ & $0.420 \pm 0.056$ \\
 & NCP & $0.630 \pm 0.037$ & $1.3 \pm 0.0$ & $0.029 \pm 0.001$ & $0.420 \pm 0.056$ \\
 & CCP & $0.604 \pm 0.049$ & $5.1 \pm 0.8$ & $0.027 \pm 0.002$ & $0.282 \pm 0.068$ \\
\end{longtable}

\begin{table}[t]
\centering
\caption{Within-family comparison of CFCP against baselines at $\alpha=0.1$. Values are means over five repeated dataset splits. \textbf{Bold} marks the best under coverage error within each family. Full results including 95\% CIs are in Appendix~\ref{app:results}. }
\label{tab:main_results_err}
\small
\setlength{\tabcolsep}{4pt}
\resizebox{\linewidth}{!}{%
\begin{tabular}{l l cc cc cc cc cc cc}
\toprule
 &  & \multicolumn{2}{c}{\textbf{CFCP (ours)}} & \multicolumn{2}{c}{Split} & \multicolumn{2}{c}{ICP} & \multicolumn{2}{c}{CCP} & \multicolumn{2}{c}{NCP} & \multicolumn{2}{c}{RC3P} \\
\cmidrule(lr){3-4}\cmidrule(lr){5-6}\cmidrule(lr){7-8}\cmidrule(lr){9-10}\cmidrule(lr){11-12}\cmidrule(lr){13-14}
\textbf{Dataset} & \textbf{Score} & WUC$\downarrow$ & MaxCE$\downarrow$ & WUC$\downarrow$ & MaxCE$\downarrow$ & WUC$\downarrow$ & MaxCE$\downarrow$ & WUC$\downarrow$ & MaxCE$\downarrow$ & WUC$\downarrow$ & MaxCE$\downarrow$ & WUC$\downarrow$ & MaxCE$\downarrow$ \\
\midrule
\multirow{4}{*}{CIFAR-100} & LAC & 0.030 & 0.196 & 0.033 & 0.228 & 0.026 & 0.164 & 0.027 & 0.204 & \textbf{0.023} & 0.164 & 0.026 & 0.164 \\
 & APS & 0.025 & 0.220 & 0.032 & 0.212 & 0.026 & 0.188 & 0.029 & 0.188 & \textbf{0.024} & 0.172 & 0.025 & 0.196 \\
 & RAPS & 0.031 & 0.212 & 0.036 & 0.252 & \textbf{0.025} & 0.172 & 0.027 & 0.196 & 0.027 & 0.204 & \textbf{0.025} & 0.172 \\
 & SAPS & 0.033 & 0.212 & 0.036 & 0.244 & \textbf{0.024} & 0.188 & 0.028 & 0.204 & 0.027 & 0.196 & 0.027 & 0.188 \\
\midrule
\multirow{4}{*}{ImageNet} & LAC & 0.058 & 0.520 & 0.060 & 0.540 & \textbf{0.038} & 0.484 & 0.040 & 0.556 & 0.063 & 0.500 & 0.049 & 0.516 \\
 & APS & 0.030 & 0.380 & 0.023 & 0.392 & \textbf{0.019} & 0.240 & 0.020 & 0.276 & 0.033 & 0.448 & 0.023 & 0.424 \\
 & RAPS & 0.042 & 0.440 & 0.039 & 0.404 & \textbf{0.030} & 0.352 & 0.034 & 0.400 & 0.047 & 0.500 & 0.033 & 0.444 \\
 & SAPS & 0.046 & 0.480 & 0.055 & 0.476 & \textbf{0.031} & 0.440 & 0.033 & 0.436 & 0.049 & 0.448 & 0.032 & 0.444 \\
\midrule
\multirow{4}{*}{ImageNet-V2} & LAC & 0.128 & 0.900 & 0.128 & 0.900 & 0.132 & 0.740 & 0.133 & 0.800 & 0.130 & 0.800 & \textbf{0.125} & 0.720 \\
 & APS & 0.046 & 0.500 & 0.043 & 0.520 & \textbf{0.042} & 0.440 & \textbf{0.042} & 0.460 & 0.048 & 0.600 & 0.047 & 0.560 \\
 & RAPS & 0.092 & 0.700 & 0.095 & 0.700 & 0.088 & 0.680 & \textbf{0.087} & 0.680 & 0.094 & 0.700 & \textbf{0.087} & 0.720 \\
 & SAPS & 0.122 & 0.900 & 0.123 & 0.820 & \textbf{0.104} & 0.720 & 0.108 & 0.740 & 0.125 & 0.800 & 0.116 & 0.740 \\
\midrule
\multirow{4}{*}{WOS-46985} & LAC & 0.029 & 0.710 & 0.035 & 0.510 & 0.021 & 0.312 & 0.027 & 0.324 & 0.027 & 0.810 & \textbf{0.020} & 0.312 \\
 & APS & 0.038 & 0.837 & 0.023 & 0.300 & \textbf{0.020} & 0.278 & 0.022 & 0.240 & 0.036 & 0.475 & 0.023 & 0.262 \\
 & RAPS & 0.038 & 0.837 & 0.021 & 0.300 & \textbf{0.020} & 0.243 & 0.021 & 0.205 & 0.036 & 0.503 & \textbf{0.020} & 0.261 \\
 & SAPS & 0.048 & 0.827 & 0.029 & 0.420 & \textbf{0.021} & 0.257 & 0.027 & 0.282 & 0.029 & 0.420 & \textbf{0.021} & 0.270 \\
\bottomrule
\end{tabular}%
}
\end{table}

\begin{table}
\centering
\caption{Example CFCP hyperparameter grid search used in the sensitivity analysis. The table reports top 20 APS configurations across locality-specific hyperparameters on the CIFAR-100 dataset: number of clusters $K$, number of neighboring centroids $m$, assignment temperature $\tau$, Dirichlet smoothing $\beta$, reliability exponent $\gamma$, and support shrinkage $\beta_{\mathrm{sup}}$. Displayed are our main metrics: missed class coverage (Misses) and average set size, as well as the secondary diagnostics of weighted under-coverage (WUC) and MaxCE. Lower is better for all metrics. The near-best configurations achieve similar class coverage and set size across different hyperparameter choices, illustrating that CFCP is task-adaptive but not brittle to a single finely tuned setting.}
\label{tab:hpo_grid}
\begin{tabular}{r r r r r r lr r r }\toprule
 $K$ & $m$ & $\tau$ & $\beta$ & $\gamma$ & $\beta_{sup}$ &  Misses&Size & WUC & MaxCE \\\midrule
 120 & 3 & 0.08 & 2 & 2 & 150 &  45.4 &8.368 & 0.028 & 0.204 \\
 120 & 3 & 0.08 & 2 & 1 & 150 &  46.6 &7.851 & 0.028 & 0.188 \\
 120 & 3 & 0.12 & 2 & 2 & 150 &  45.4 &8.330 & 0.028 & 0.188 \\
 120 & 3 & 0.08 & 8 & 2 & 150 &  43.4 &9.713 & 0.028 & 0.196 \\
 120 & 3 & 0.12 & 2 & 1 & 150 &  45.4 &7.614 & 0.029 & 0.188 \\
 120 & 3 & 0.08 & 8 & 1 & 40 &  45.6 &8.927 & 0.029 & 0.180 \\
 120 & 3 & 0.08 & 2 & 1 & 40 &  46.6 &8.451 & 0.029 & 0.196 \\
 120 & 20 & 0.08 & 2 & 1 & 40 &  44.8 &8.566 & 0.029 & 0.236 \\
 120 & 3 & 0.08 & 8 & 2 & 40 &  45.0 &11.973 & 0.029 & 0.180 \\
 120 & 3 & 0.08 & 8 & 1 & 150 &  46.4 &8.041 & 0.029 & 0.180 \\
 120 & 3 & 0.12 & 8 & 2 & 150 &  46.8 &10.038 & 0.029 & 0.212 \\
 120 & 3 & 0.08 & 2 & 2 & 40 &  45.6 &9.249 & 0.030 & 0.172 \\
 120 & 20 & 0.08 & 2 & 1 & 150 &  46.2 &7.758 & 0.030 & 0.220 \\
 120 & 3 & 0.12 & 8 & 1 & 40 &  44.8 &8.696 & 0.030 & 0.188 \\
 80 & 3 & 0.08 & 2 & 2 & 0 &  43.6 &10.778 & 0.030 & 0.212 \\
 120 & 10 & 0.08 & 2 & 1 & 150 &  46.6 &7.498 & 0.030 & 0.204 \\
 120 & 3 & 0.12 & 8 & 1 & 0 &  43.8 &10.361 & 0.030 & 0.204 \\
 80 & 3 & 0.08 & 2 & 1 & 40 &  44.6 &8.876 & 0.030 & 0.196 \\ \bottomrule

\end{tabular}

\end{table}

\begin{table}
\centering
\caption{
ImageNet and ImageNet-V2 ablation of CFCP locality. 
We compare the full CFCP variant using learned clusters and soft assignment ($K=180$, $m=50$), a hard-assignment variant using only the nearest cluster ($K=180$, $m=1$), and a no-cluster variant that uses only the global label-frequency model with shrinkage ($K=1$, $m=1$). 
Coverage misses denotes the mean number of classes whose empirical class-conditional coverage falls below the target level. Lower is better. 
Size denotes the mean prediction-set size. Lower is better. 
The results show that soft cluster assignment gives the strongest coverage-efficiency tradeoff, while removing clusters produces very large prediction sets.}
\label{tab:ablation}
\setlength{\tabcolsep}{20pt}
% \resizebox{200pt}{!}{%
\begin{tabular}{l l r r}\toprule
Ablation & Score & Misses& Set Size\\\midrule
\multicolumn{4}{@{}c@{}}{\cellcolor{gray!15}\small \textbf{ImageNet}} \\
\textbf{Full CFCP}& APS & 328 & 15.1 \\
& LAC & 454.4 & 2.1 \\
& RAPS & 326.6 & 4.5 \\
& SAPS & 430.6 & 2.6 \\
\midrule
\textbf{Hard assignment}& APS & 439.2 & 29.6 \\
& LAC & 430.8 & 15.5 \\
& RAPS & 442.6 & 15.7 \\
& SAPS & 437.2 & 15.8 \\
\midrule
\textbf{No clusters / Global}& APS & 397.2 & 340.0 \\
& LAC & 396.8 & 340.3 \\
& RAPS & 397.2 & 340.0 \\
& SAPS & 398 & 340.1 \\
\multicolumn{4}{@{}c@{}}{\cellcolor{gray!15}\small \textbf{ImageNet-V2}} \\
\textbf{Full CFCP}&  APS & 286.4 & 31.2 \\
& LAC & 556.2 & 2.4 \\
& RAPS & 433 & 5.1 \\
& SAPS & 546.2 & 2.9 \\
\midrule
\textbf{Hard assignment}& APS & 543.4 & 30.8 \\
& LAC & 545.8 & 15.6 \\
& RAPS & 557 & 15.8 \\
& SAPS & 552 & 15.9 \\
\midrule
\textbf{No clusters / Global}& APS & 446.6 & 340.1 \\
& LAC & 447.2 & 340.4 \\
& RAPS & 446.6 & 340.1 \\
& SAPS & 446.2 & 340.1 \\ \bottomrule

\end{tabular}
% }
\end{table}

\subsection{Failure-mode case study on CIFAR-100}
\label{app:failure}

To better understand residual tail failures, we compare CFCP against the closest local-neighborhood baseline, NCP, on CIFAR-100 using APS and RAPS. For each score family, we identify the worst-covered CFCP classes and report their empirical class-conditional coverage, coverage deficit relative to the nominal target $1-\alpha=0.9$, class frequency and average prediction-set size. 
Since CIFAR-100 is class-balanced, all reported classes have the same empirical frequency. Therefore, these failures are not explained by class rarity, but rather by residual representation ambiguity or by the smoothing and global-thresholding components of CFCP.

\begin{table}[ht!]
\centering
\setlength{\tabcolsep}{4.5pt}
\caption{
CIFAR-100 failure-mode case study comparing CFCP and NCP at $\alpha=0.1$.
Rows show the worst-covered CFCP classes for APS and RAPS, using mean classwise coverage over five splits.
$\Delta_{0.9}=[0.9-\widehat{\mathrm{Cov}}_c]_+$ denotes the classwise under-coverage deficit.
Class names follow the standard CIFAR-100 class ordering.
}
\label{tab:cifar100_failure_modes}
\begin{tabular}{llccccccc}
\toprule
Score & Class & Freq. &
\multicolumn{2}{c}{Coverage} &
\multicolumn{2}{c}{$\Delta_{0.9}$} &
\multicolumn{2}{c}{Set size} \\
\cmidrule(lr){4-5}
\cmidrule(lr){6-7}
\cmidrule(lr){8-9}
& & & CFCP & NCP & CFCP & NCP & CFCP & NCP \\
\midrule
\multirow{3}{*}{APS}
& class94 (turtle)   & 1.0\% & 0.744 & 0.816 & 0.156 & 0.084 & 12.90 & 19.62 \\
& class65 (possum)   & 1.0\% & 0.784 & 0.864 & 0.116 & 0.036 & 11.50 & 20.45 \\
& class32 (elephant) & 1.0\% & 0.808 & 0.904 & 0.092 & 0.000 &  8.56 & 13.03 \\
\midrule
\multirow{3}{*}{RAPS}
& class94 (turtle)  & 1.0\% & 0.736 & 0.808 & 0.164 & 0.092 & 9.46 & 11.56 \\
& class66 (rabbit)  & 1.0\% & 0.768 & 0.872 & 0.132 & 0.028 & 9.50 & 11.97 \\
& class65 (possum)  & 1.0\% & 0.768 & 0.776 & 0.132 & 0.124 & 9.38 & 12.49 \\
\bottomrule
\end{tabular}
\end{table}

Table~\ref{tab:cifar100_failure_modes} shows that CFCP can improve the broad coverage-efficiency tradeoff while still leaving a small number of classes with larger residual under-coverage. 
The clearest failure case is class94, which is the worst-covered CFCP class under both APS and RAPS. For APS, CFCP attains substantially smaller sets than NCP for the same classes, but this efficiency is accompanied by a larger worst-class deficit. 
For RAPS, the same class remains the largest residual failure, although the difference in set size is smaller. 
Both CFCP and NCP identify the classifier's low confidence on samples from the under-covered classes, resulting in larger prediction sets. Still, for these classes CFCP exhibits worse coverage due to its residual limitation.
These results motivate tail-sensitive extensions of CFCP, such as class or cluster-aware threshold adjustments, while preserving the main empirical advantage of the method: improved class coverage and competitive or smaller prediction sets on average.

\begin{figure}[ht!]
\centering
\begin{subfigure}{\linewidth}
    \centering
    \includegraphics[width=\linewidth]{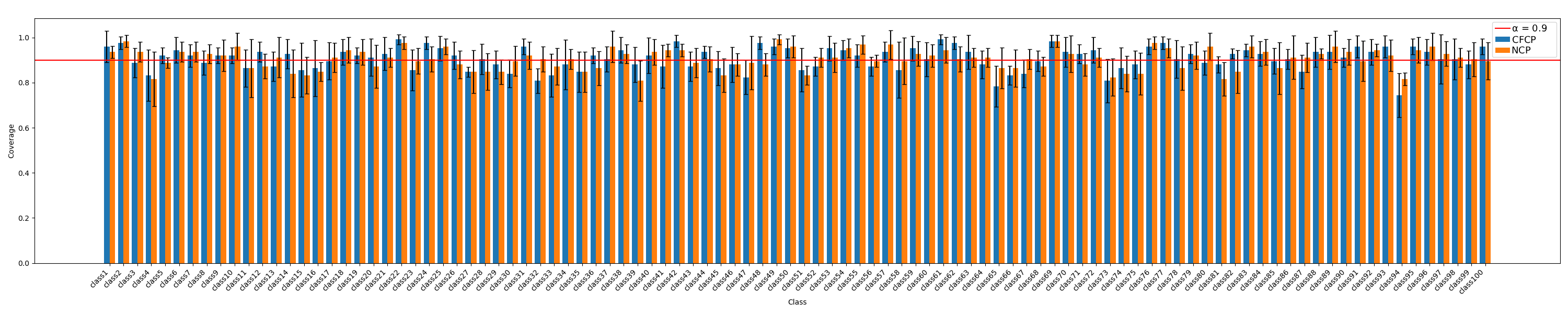}
    \caption{Class-conditional coverage.}
    \label{fig:cifar100_aps_coverage}
\end{subfigure}

\vspace{0.6em}

\begin{subfigure}{\linewidth}
    \centering
    \includegraphics[width=\linewidth]{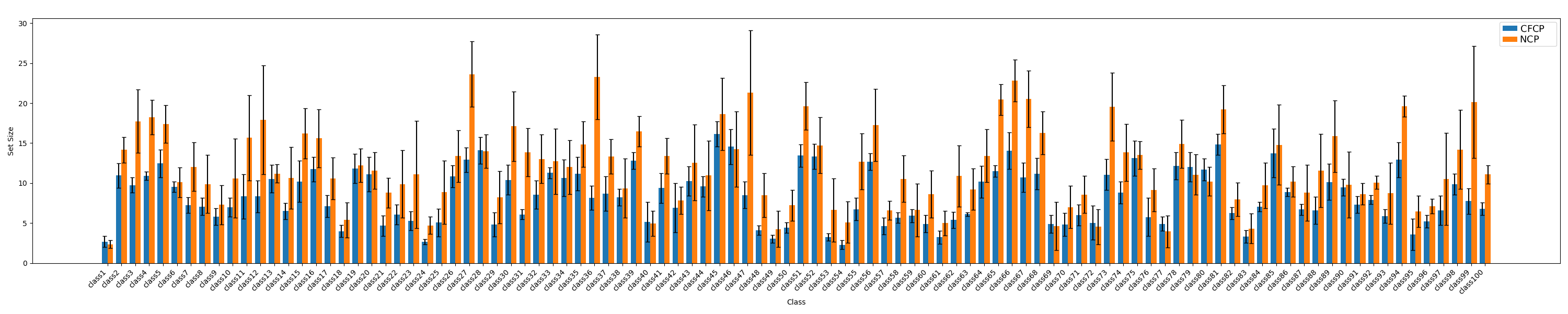}
    \caption{Class-conditional average prediction-set size.}
    \label{fig:cifar100_aps_size}
\end{subfigure}

\caption{
CIFAR-100 classwise comparison of CFCP-APS and NCP-APS at $\alpha=0.1$.
The red horizontal line marks the nominal target coverage $1-\alpha=0.9$.
CFCP is generally more efficient, but the failure-mode table shows that a few classes retain larger under-coverage deficits.
}
\label{fig:cifar100_aps_failure_modes}
\end{figure}

\begin{figure}[t]
\centering
\begin{subfigure}{\linewidth}
    \centering
    \includegraphics[width=\linewidth]{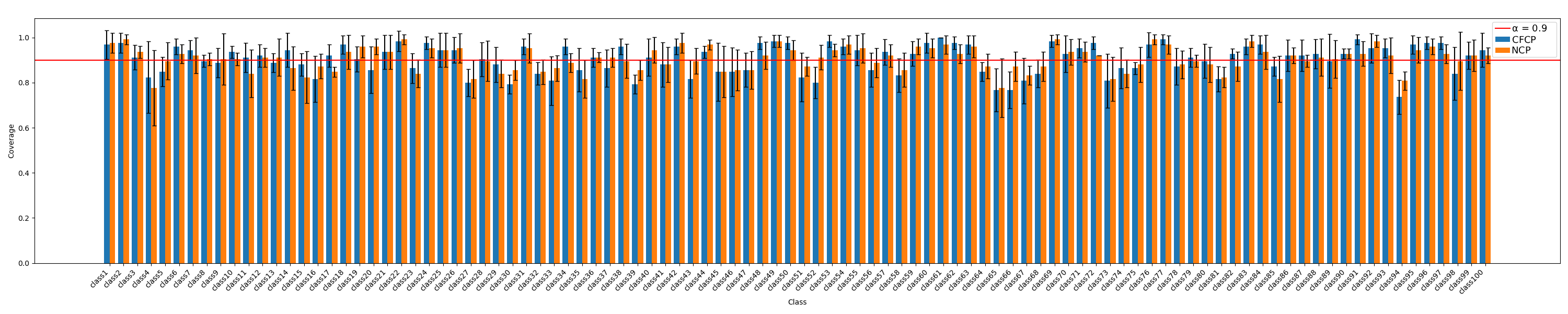}
    \caption{Class-conditional coverage.}
    \label{fig:cifar100_raps_coverage}
\end{subfigure}

\vspace{0.6em}

\begin{subfigure}{\linewidth}
    \centering
    \includegraphics[width=\linewidth]{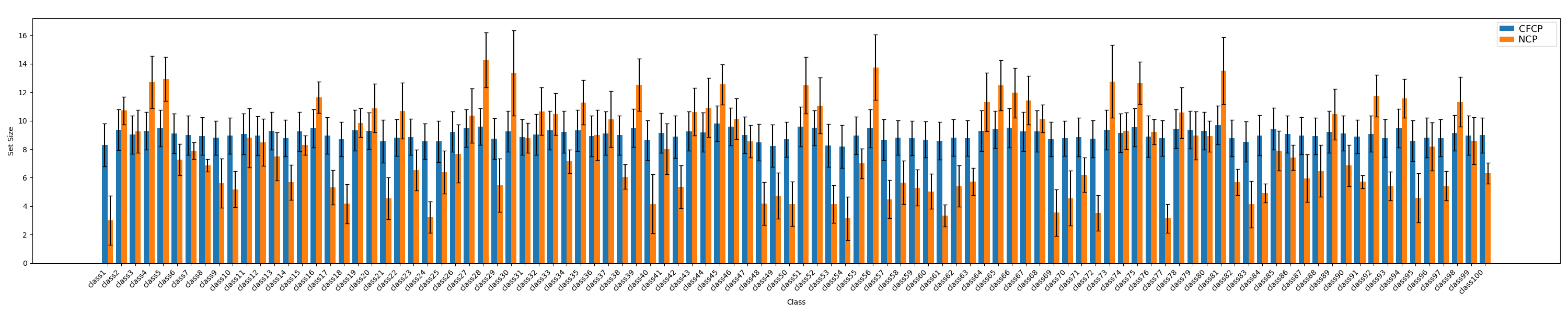}
    \caption{Class-conditional average prediction-set size.}
    \label{fig:cifar100_raps_size}
\end{subfigure}

\caption{
CIFAR-100 classwise comparison of CFCP-RAPS and NCP-RAPS at $\alpha=0.1$.
CFCP covers more classes in the mean classwise view, but the worst-covered CFCP class still has a larger residual under-coverage deficit than the corresponding NCP value.
}
\label{fig:cifar100_raps_failure_modes}
\end{figure}

\clearpage

\section*{NeurIPS Paper Checklist}

\begin{enumerate}

\item {\bf Claims}
    \item[] Question: Do the main claims made in the abstract and introduction accurately reflect the paper's contributions and scope?
    \item[] Answer: \answerYes{}
    \item[] Justification: The abstract and introduction state the paper’s three main contributions: introducing CFCP as a plug-in conformal framework, establishing finite-sample marginal validity in the disjoint-split regime together with an oracle local-validity interpretation for randomized APS, and showing strong empirical class-coverage performance on image and text benchmarks. The paper also states important limitations, including that worst-class max coverage error can remain high on rare classes.

\item {\bf Limitations}
    \item[] Question: Does the paper discuss the limitations of the work performed by the authors?
    \item[] Answer: \answerYes{}
    \item[] Justification: Section~\ref{sec:dicussion} discusses limitations. It notes that CFCP smooths across nearby clusters with a single global threshold, which can dilute very rare or atypical classes, leading to worse max coverage error on a small number of hard classes. Section~\ref{sec:dicussion} also outlines future directions including tail-sensitive extensions. A Failure-mode case study is available in Appendix \ref{app:failure}.

\item {\bf Theory assumptions and proofs}
    \item[] Question: For each theoretical result, does the paper provide the full set of assumptions and a complete (and correct) proof?
    \item[] Answer: \answerYes{}
    \item[] Justification: Section~\ref{sec:theoretical-analysis} states the main theoretical results (Theorem~\ref{thm:cfcp-marginal}, Corollary~\ref{cor:cfcp-marginal}, Corollary~\ref{cor:cfcp-oracle-local-main}) with their assumptions clearly listed. Full proofs are provided in Appendix \ref{app:validity}, including the marginal validity proof (Appendix~\ref{app:marginal-validity}, Theorem~\ref{thm:cfcp-marginal} and Corollary~\ref{cor:cfcp-marginal}), the oracle local validity proof for randomized APS (Appendix~\ref{app:cfcp-oracle-local-main}, Lemma~\ref{lem:aps-uniform} and Theorem~\ref{thm:oracle-local}), and the approximate local validity proposition (Appendix~\ref{app:approx-local}, Proposition~\ref{prop:approx-local}). Theorem~\ref{thm:oracle-local} and Corollary~\ref{cor:cfcp-oracle-local-main} provide the oracle information advantage result. All theorems and lemmas are numbered and cross-referenced.

    \item {\bf Experimental result reproducibility}
    \item[] Question: Does the paper fully disclose all the information needed to reproduce the main experimental results of the paper to the extent that it affects the main claims and/or conclusions of the paper (regardless of whether the code and data are provided or not)?
    \item[] Answer: \answerYes{}
    \item[] Justification: Section~\ref{sec:experiments} specifies all datasets, architectures, baseline methods and coverage level. Appendix~\ref{app:experimental_setup} provides full experimental details on method parameters, data splits, model training (including optimizer settings, learning rates, and epochs), and baseline hyperparameter grids. The paper also provides the code for reproducing all experiments. All results are reported as means over five repeated dataset splits with fixed seeds.

\item {\bf Open access to data and code}
    \item[] Question: Does the paper provide open access to the data and code, with sufficient instructions to faithfully reproduce the main experimental results, as described in supplemental material?
    \item[] Answer: \answerYes{}
    \item[] Justification: Code for reproducing all experiments and summary results is available at an anonymous repository. The repository includes a README with instructions and a license. All datasets used (CIFAR-100, ImageNet, ImageNet-V2, WOS-46985) are publicly available benchmarks.

\item {\bf Experimental setting/details}
    \item[] Question: Does the paper specify all the training and test details (e.g., data splits, hyperparameters, how they were chosen, type of optimizer) necessary to understand the results?
    \item[] Answer: \answerYes{}
    \item[] Justification: Section~\ref{sec:experiments} describes datasets, model architectures, learned representations, baseline methods, coverage level and statistical significance methodology. Appendix~\ref{app:experimental_setup} provides the complete experimental setup for each dataset, including method parameters, data splits, model architectures, pre-training strategies, fine-tuning hyperparameters (learning rate, weight decay, warmup ratio, epochs), and CFCP-specific hyperparameter grids.

\item {\bf Experiment statistical significance}
    \item[] Question: Does the paper report error bars suitably and correctly defined or other appropriate information about the statistical significance of the experiments?
    \item[] Answer: \answerYes{}
    \item[] Justification: All reported results are means over five repeated dataset splits with fixed seeds and 95\% confidence intervals are computed (Section~\ref{sec:experiments}). The 95\% CIs are omitted from the main tables for readability and are available in Appendix \ref{app:results} with the full experimental results.

\item {\bf Experiments compute resources}
    \item[] Question: For each experiment, does the paper provide sufficient information on the computer resources (type of compute workers, memory, time of execution) needed to reproduce the experiments?
    \item[] Answer: \answerYes{}
    \item[] Justification: Appendix~\ref{app:experimental_setup} states that all experiments were performed on a MacBook Pro M4 Max with 36GB RAM. Section~\ref{subsec:complexity} and Appendix~\ref{app:complexity} discuss the computational complexity of CFCP, including the time complexity of clustering and nearest-centroid retrieval. A full run of all reported CFCP experiments required approximately 3-4 hours. Individual dataset runs required approximately 30-60 minutes depending on the dataset.
    
\item {\bf Code of ethics}
    \item[] Question: Does the research conducted in the paper conform, in every respect, with the NeurIPS Code of Ethics \url{https://neurips.cc/public/EthicsGuidelines}?
    \item[] Answer: \answerYes{}
    \item[] Justification: The research conforms with the NeurIPS Code of Ethics. The paper uses only publicly available benchmark datasets, does not involve human subjects, and the proposed method is a general-purpose improvement to conformal prediction that does not raise specific ethical concerns.

\item {\bf Broader impacts}
    \item[] Question: Does the paper discuss both potential positive societal impacts and negative societal impacts of the work performed?
    \item[] Answer: \answerYes{}
    \item[] Justification: CFCP may improve reliability in high-stakes classification settings by reducing classwise coverage disparities. However, it does not guarantee exact coverage for every subgroup, and the failure-mode analysis shows that some hard classes may remain under-covered. Therefore, CFCP should not be used as a standalone safety guarantee in medical, legal, or other high-stakes deployments without domain-specific validation and monitoring.
    
\item {\bf Safeguards}
    \item[] Question: Does the paper describe safeguards that have been put in place for responsible release of data or models that have a high risk for misuse (e.g., pre-trained language models, image generators, or scraped datasets)?
    \item[] Answer: \answerNA{}
    \item[] Justification: The paper does not release pre-trained models, generated data, or scraped datasets. CFCP is a plug-in conformal prediction method that operates on top of existing classifiers and uses only publicly available benchmark datasets. The released code poses no risk for misuse.

\item {\bf Licenses for existing assets}
    \item[] Question: Are the creators or original owners of assets (e.g., code, data, models), used in the paper, properly credited and are the license and terms of use explicitly mentioned and properly respected?
    \item[] Answer: \answerYes{}
    \item[] Justification: The original papers for all datasets (CIFAR-100, ImageNet, ImageNet-V2, WOS-46985), pre-trained models (ResNet-50, DistilBERT, SimCLR-v2), and software libraries (PyTorch, TorchCP, FAISS) are properly cited, including license indication.

\item {\bf New assets}
    \item[] Question: Are new assets introduced in the paper well documented and is the documentation provided alongside the assets?
    \item[] Answer: \answerYes{}
    \item[] Justification: The released code repository includes a README with usage instructions and a license file. The paper provides the full algorithmic specification in Section~\ref{sec:method} and Appendix~\ref{app:method}, which together with the code enable reproduction of all results.

\item {\bf Crowdsourcing and research with human subjects}
    \item[] Question: For crowdsourcing experiments and research with human subjects, does the paper include the full text of instructions given to participants and screenshots, if applicable, as well as details about compensation (if any)?
    \item[] Answer: \answerNA{}
    \item[] Justification: The paper does not involve crowdsourcing or research with human subjects.

\item {\bf Institutional review board (IRB) approvals or equivalent for research with human subjects}
    \item[] Question: Does the paper describe potential risks incurred by study participants, whether such risks were disclosed to the subjects, and whether Institutional Review Board (IRB) approvals (or an equivalent approval/review based on the requirements of your country or institution) were obtained?
    \item[] Answer: \answerNA{}
    \item[] Justification: The paper does not involve crowdsourcing or research with human subjects.

\item {\bf Declaration of LLM usage}
    \item[] Question: Does the paper describe the usage of LLMs if it is an important, original, or non-standard component of the core methods in this research? Note that if the LLM is used only for writing, editing, or formatting purposes and does \emph{not} impact the core methodology, scientific rigor, or originality of the research, declaration is not required.
    %this research?
    \item[] Answer: \answerNA{}
    \item[] Justification: The core method development in this research does not involve LLMs as any important, original, or non-standard component. CFCP operates on learned representations from standard classifiers (ResNet-50, DistilBERT, SimCLR-v2) and does not use LLMs in its methodology.

\end{enumerate}

\end{document}